\documentclass{article}

\usepackage[preprint]{neurips_2026}

\usepackage{mathtools}
\usepackage[utf8]{inputenc} 
\usepackage[T1]{fontenc}    
\usepackage{hyperref}       
\usepackage{url}            
\usepackage{booktabs}       
\usepackage{amsfonts}       
\usepackage{nicefrac}       
\usepackage{microtype}      
\usepackage{xcolor}         
\definecolor{gold}{RGB}{255,192,0}
\definecolor{blue}{RGB}{68,114,196}
\definecolor{pink}{RGB}{241,158,155}
\definecolor{cyan}{RGB}{97,211,175}
\usepackage{caption}
\usepackage{graphicx}
\usepackage{physics}
\usepackage{enumitem}
\usepackage{amssymb}
\usepackage{amsthm}
\usepackage{wrapfig}
\usepackage{algorithm}
\usepackage{algorithmic}
\usepackage{tabularx}
\usepackage{makecell}
\usepackage{multirow}
\usepackage{multicol}
\theoremstyle{plain}
\newtheorem{theorem}{Theorem}[section]

\newtheorem{lemma}[theorem]{Lemma}

\theoremstyle{definition}
\newtheorem{definition}[theorem]{Definition}

\theoremstyle{remark}

\newcommand{\shorteq}[1]{\mathrel{\makebox[#1][c]{=}}}
\newcolumntype{Y}{>{\centering\arraybackslash}X}

\title{Geometry-Calibrated Conformal Abstention for Language Models}

\author{%
  Rui Xu,\,  Yi Chen,\,  Sihong Xie,\,  Hui Xiong\\
  Information Hub, AI Thrust\\
  Hong Kong University of Science and Technology (Guangzhou)\\
  Guangzhou, Guangdong, China\\
}

\begin{document}

\maketitle

\begin{abstract}
When language models lack relevant knowledge for a given query, they frequently generate plausible responses that can be hallucinations, rather than admitting being agnostic about the answer.
Retraining models to reward admitting ignorance can lead to overly conservative behaviors and poor generalization due to scarce evaluation benchmarks.
We propose a post-hoc framework, \emph{Conformal Abstention} (CA), adapted from conformal prediction (CP) to determine whether to abstain from answering a query. CA provides finite-sample guarantees on both the probability of participation (i.e., not abstaining) and the probability that the generated response is correct. Importantly, the abstention decision relies on prediction confidence rather than the non-conformity scores used in CP, which are intractable for open-ended generation. To better align prediction confidence with the model’s ignorance, we introduce a calibration strategy using representation geometry within the model to measure knowledge involvement in shaping the response.
Experiments demonstrate we improve selective answering significantly with 75$\%$ conditional correctness.
\end{abstract}

\section{Introduction}
Advances in model architectures~\cite{vaswani2017attention} have endowed language models with broad linguistic competence. Leveraging the ability, large language models (LLMs) have become central components of artificial intelligence, serving as general-purpose engines for reasoning~\cite{stechly2024chain}, knowledge retrieval~\citep{zhang2024comprehensive}, and decision support~\cite{eigner2024determinants}, with growing deployment in science~\cite{reddy2025towards}, healthcare~\cite{yang2024talk2care}, and education~\cite{neumann2024llm}.

Despite their remarkable capabilities, LLMs should refuse to respond in certain situations, such as queries involving safety~\cite{mu2024rulebasedrewardslanguage} or privacy~\cite{wang2025unveiling}. Beyond externally imposed constraints, LLMs must also recognize the limits of their own knowledge~\cite{wen2025know}. When faced with questions for which they do not know the answer, models tend to generate responses without reliable grounding, producing responses that appear plausible but are in fact hallucinated. This behavior raises persistent concerns about reliability in real-world deployments~\cite{huang2025survey}. Importantly, this phenomenon is not solely due to limitations in model capacity, but is closely tied to commonly used training loss functions~\cite{kalai2025languagemodelshallucinate}. Most evaluation protocols reward models only for producing an answer, typically measuring performance through accuracy. Consequently, abstention (e.g., by responding "I don't know") receives zero or even negative rewards, while guessing always offers a nonzero chance of success, making guessing more rewarding than abstaining for all queries.   
Recent work suggests that rewarding appropriate expressions of ignorance, rather than penalizing them, may help mitigate this issue~\cite{li2025verifiableaccuracyabstentionrewards}. 

However, the scarcity of benchmarks that explicitly reward abstention causes models to be overfitting, which in turn degrades their generalization ability.~\cite{lin2024wildbenchbenchmarkingllmschallenging}. Moreover, excessively rewarding model refusal can induce indiscriminate abstention~\cite{kalai2025languagemodelshallucinate}.
As a result, promoting abstention through training-based interventions is can be fragile.
These considerations motivate a central question:
\begin{center}
\vspace{-5pt}
\textit{Can LLMs abstain using a post-hoc mechanism?}
\vspace{-5pt}
\end{center}
We propose \textbf{Conformal Abstention (CA)}, an approach grounded in conformal prediction (CP)~\cite{vovk2005algorithmic} that enables LLMs to admit a lack of knowledge. Specifically, CA establishes two theoretical finite-sample guarantees: (i) a \textit{participation guarantee}, bounding the probability that the model participates (i.e., an LLM provides an informative answer) rather than saying “I don’t know”; and (ii) a \textit{conditional correctness guarantee}, which bounds the probability that the provided answer is correct. The abstention decision is based on prediction confidence, avoiding the nonconformity scores used in CP, which are empirically intractable for open-ended generation~\cite{campos2024conformal}.
\begin{figure}[h]
\centering
\captionsetup{singlelinecheck = false, skip=5pt, justification=justified}
  \includegraphics[scale=0.40]{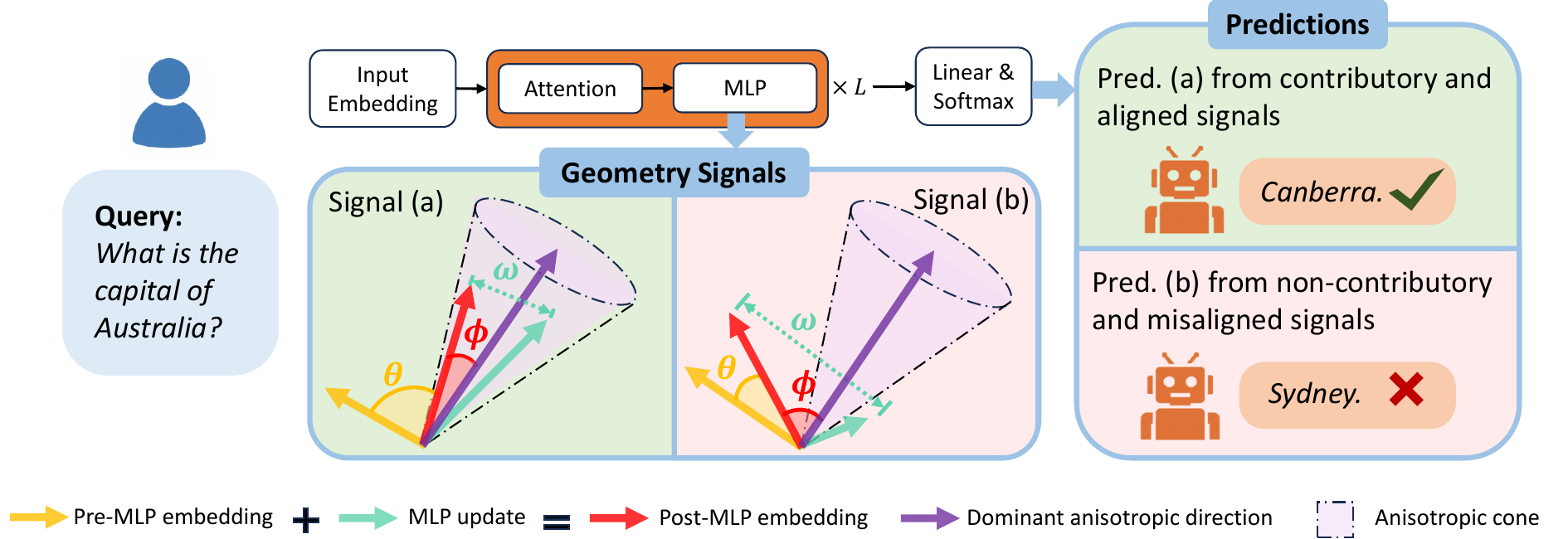}
  \caption{
  Token representation geometry captures how the model's internal knowledge shapes its response.}
  \label{fig: case study} 
  \vspace{-10pt}
\end{figure}

To derive prediction confidence that faithfully reflects when the model is truly uncertain, we analyze how internal knowledge influences the generation process through the trajectory of token representations across layers. 
Specifically, we focus on the update induced by the MLP, which is known to encode the model’s factual knowledge. As illustrated in Figure~\ref{fig: case study}, we quantify this influence using three geometric signals:
\begin{enumerate}[topsep=0pt, itemsep=1pt, leftmargin=15pt]
    \item \textcolor{cyan}{\textbf{proximity} $\omega$}, measuring the MLP's contribution via its distance to the post-MLP embedding;
    \item \textcolor{gold}{\textbf{embedding rotation} $\theta$}, capturing the extent to which the update alters the prediction direction;
    \item \textcolor{red}{\textbf{alignment} $\phi$}, measuring consistency with the dominant anisotropic direction.
\end{enumerate}
Intuitively, when the MLP update contributes strongly (small distance to the post-MLP embedding), induces a significant rotation, and aligns well with the dominant direction, the model is more likely to leverage useful internal knowledge, leading to a correct prediction (Figure~\ref{fig: case study} signal and pred. (a)). Conversely, weak contribution or misalignment indicates uncertainty, where the response is less reliable (Figure~\ref{fig: case study}  signal and pred. (b)). We leverage these geometric signals to calibrate confidence.

Across six datasets, the proposed method consistently outperforms a broad set of strong baselines, achieving the highest average conditional correctness of 75.0$\%$ under conformal abstention.
\section{Background}
\subsection{Conformal Prediction for Language Models}
We begin with a brief review of conformal prediction (CP)~\cite{angelopoulos2020uncertainty}. Throughout, we use uppercase letters (e.g.,$X$) to denote random variables, lowercase letters (e.g., $x$) to denote their realizations, and script letters (e.g., $\mathcal{X}$) to denote sets, unless stated otherwise. Given a model $f: \mathcal{X}\rightarrow\mathcal{Y}$, a score function $g:\mathcal{X}\times\mathcal{Y}\rightarrow\mathcal{S}\subseteq\mathbb{R}$ outputs non-conformity scores to assess how data conforms to the model $f$, with lower values indicating better conformity. For classification tasks, it can be defined as the negative predicted probability of the true label. With calibration instances $\{(X_i, Y_i)\}_{i=1}^n$, split conformal prediction computes calibration non-conformity scores $S_i := g(X_i, Y_i)$ for $i=1,\dots,n$~\citep{papadopoulos2002inductive}.  Let $\tau$ be the $\lceil(1-\alpha)(n+1)\rceil/n$ quantile of $\{S_i\}_{i=1}^n$.\footnote{\( \tau \) is the \( \lceil (1-\alpha)(n+1) \rceil \)-th smallest value among \( \{ S_i \}_{i=1}^n \).}
For a test instance $(X_{n+1}, Y_{n+1})$, the corresponding prediction set is
$
\mathcal{C}(X_{n+1}) := \{y \in \mathcal{Y} : g(X_{n+1}, y) \leq \tau\}
$. 
Assuming exchangeability, we can ensure
\begin{equation}\label{eq:marginal-guarantee}
    \mathbb{P}\big(Y_{n+1} \in \mathcal{C}(X_{n+1})\big) = \mathbb{P}(S_{n+1} \leq \tau) \geq 1-\alpha.
\end{equation}
However, in autoregressive language generation, ground-truth answers can take semantically equivalent forms, making it intractable to define the score function $g$ via likelihood over all valid outputs.

To address this, prior work restricts generation to structured settings such as multiple-choice classification~\cite{kumar2023conformal, ren2023robots, kostumov2024uncertainty} or regression-based grading~\cite{sheng2025analyzinguncertaintyllmasajudgeinterval}, which limits applicability to open-ended tasks. Another line of work approximates nonconformity via repeated sampling, but relies on the unrealistic assumption that the ground-truth answer is sampled~\cite{wang2024sample,wang2024conu}. When this assumption fails, valid coverage can only be achieved if 
$1-\alpha$ is below the fraction of calibration instances where the correct answer appears~\cite{su2024apienoughconformalprediction, wang2025sconuselectiveconformaluncertainty}.
Alternatively, some approaches redefine correctness using semantic entailment~\cite{mohri2024language} or admissibility~\cite{quach2023conformal}, weakening the connection between conformal guarantees and true correctness, as such criteria may still admit incomplete or misleading responses.

Abstention has recently emerged as a promising alternative direction. For example, Yadkori et al.~\cite{yadkori2024mitigatingllmhallucinationsconformal} adopt conformal risk control~\cite{angelopoulos2022conformal} to bound the joint probability of model participation (i.e., not abstaining) and incorrect prediction. However, since model participation is directly observable, the conditional probability of correctness given participation is a more useful and practically relevant than the joint probability. Our theoretical framework explicitly targets this conditional guarantee.
\subsection{Estimating Reliability from Internal Model Signals}
Reliability estimation based on surface-level signals, such as prediction probabilities~\cite{fomicheva2020unsupervised} or sampling-based consistency~\cite{kuhn2023semantic, wan-etal-2025-reasoning}, is often insufficient for LLMs. Recent work instead leverages internal signals, including attention spectra~\cite{sriramanan2024llm}, hidden-state statistics~\cite{chen2024insidellmsinternalstates}, and embedding distances~\cite{vazhentsev2025tokenleveldensitybaseduncertaintyquantification}. However, these approaches overlook how internal knowledge shapes the evolution of representations. When relevant knowledge is present, it should induce greater changes in hidden states. This suggests that the geometry of representation transformations provides a meaningful signal for reliability. We validate this hypothesis by leveraging representation geometry for calibration, leading to improved alignment between confidence and prediction correctness.

\section{Conformal Abstention}
\subsection{Participation Guarantee}
Consider a large language model \( f \) and input query \( x \). Let \( \widehat{y} := f(x) =(\widehat{y}[1],...,\widehat{y}[N])\) denote the generated output consisting of $N$ tokens. We associate each query $x$ with an \emph{uncertainty score} $u:=r(\widehat{y}|x)$ to reflect the reliability of the model’s response $\widehat{y}$. Here, $r$ denotes a user-specified uncertainty scoring function that outputs a scalar correlated with prediction correctness; several such functions have been proposed in prior work~\cite{fomicheva2020unsupervised, duan2024shiftingattentionrelevancepredictive}. A natural choice is to derive uncertainty from the model's confidence in its own generation. For instance, denoting $p(\cdot|x)$ the conditional token distribution induced by the model $f$ given $x$, we employ perplexity and define
\begin{equation}\label{eq: perplexity}
    r^{\mathrm{perp}}(\widehat{y}|x) = \exp\left(- \frac{1}{N} \sum\nolimits_{t=1}^{N} \log p(\widehat{y}[t] \mid x, \widehat{y}[<t])\right)
\end{equation}
where \( \widehat{y}[<t] \) is the prefix up to the token at \( t-1 \).

We then introduce a threshold on the uncertainty score and abstain from prediction on inputs whose scores exceed this threshold. Formally, given \( n \) calibration samples \( \{(X_i, Y_i)\}_{i=1}^n \), we obtain response $\widehat{Y}_i:=f(X_i)$ and compute uncertainty score $
U_i \coloneqq r(\widehat{Y}_i|X_i)$, for $i = 1,\dots,n$.
Let \( \tau \) be the $\lceil(1-\alpha)(n+1)\rceil/n$ quantile of \( \{ U_i\}_{i=1}^n \):
\begin{equation}\label{eq: threshold}
    \tau := \operatorname{Quantile}\left(\lceil(1-\alpha)(n+1)\rceil/n, \{ U_i \}_{i=1}^n \right).
\end{equation}
A \emph{kept region} is defined as $\mathcal{K} \coloneqq \{ x : r(x) \le \tau \}$.
The model participates (e.g., generates an informative answer) instead of saying "I do not know" for samples in $\mathcal{K}$. If a test sample \( (X_{n+1},Y_{n+1}) \) is exchangeable, and $U_1,...,U_n$ are almost surly distinct, standard conformal arguments~\cite{vovk2005algorithmic} yield a \textit{participation guarantee}:
\begin{equation}\label{eq: participation guarantee}
    \mathbb{P}(X_{n+1}\in\mathcal{K})=\mathbb{P}(U_{n+1}\leq\tau)\in \left[1-\alpha,1-\alpha+{1}/{n+1}\right).
\end{equation}
\subsection{Conditional Correctness Guarantee}
To determine whether a prediction $\widehat{y}$ for a prompt $x$ is correct, we define an \emph{evaluation function} $\xi$ with respect to the ground-truth answer $y$, whose realization is detailed in Appendix~\ref{appendix: correctness eval}. Specifically, we have
\[
\xi(x, y, \widehat{y}) :=
\begin{cases}
1, & \text{if $\widehat{y}$ is correct for $x$ with respect to $y$},\\
0, & \text{otherwise.}
\end{cases}
\]
For notational clarity, let $J := \xi(X,Y,\widehat{Y})$ be the random variable indicating prediction correctness.
We aim to bound the probability that a prediction is correct conditioned on model participation, i.e.,
\[
\mathbb{P}(J_{n+1} = 1 \mid X_{n+1} \in \mathcal{K}).
\]
To enable finite-sample guarantees, we establish that the test uncertainty score and correctness indicator $(U_{n+1},J_{n+1})$ are exchangeable with those of the calibration samples. This follows from Theorem~\ref{thm:exchangeability preservation} showing that exchangeability is preserved under symmetric deterministic transformations, leading directly to Lemma~\ref{lem: uq_exchangeable}.

\begin{theorem}[Exchangeability Preservation]~\cite{dean1990linear}\label{thm:exchangeability preservation}
    Let $(Z_1,...,Z_{n+1})\in\mathcal{Z}^{n+1}$ be a vector of exchangeable random variables. Fix a transformation $g:\mathcal{Z}^{n+1}\rightarrow(\mathcal{Z'})^{n+1}$. If for each permutation $\pi_1$, there exists a permutation $\pi_2$ such that $\pi_1g(z_1,...,z_{n+1})=g(\pi_2(z_1,...,z_{n+1}))$ $\forall (z_1,...,z_{n+1})\in\mathcal{Z}^{n+1}$, then $(Z'_1,...,Z'_{n+1}):=g(Z_1,...Z_{n+1})$ is a vector of exchangeable random variables. 
\end{theorem}

\begin{lemma}[Exchangeability of Uncertainty and Correctness Pairs]
\label{lem: uq_exchangeable}
Let $(X_1,Y_1),\dots,(X_{n+1},Y_{n+1})$ be exchangeable random variables representing questions and ground-truth answers. Define
\[
U_i := r(f(X_i)|X_i), \quad J_i := \xi(X_i,Y_i,f(X_i)), \quad\text{for each}\quad i=1,\dots,n+1,
\]
where $r(\cdot)$ is an uncertainty scoring function,  $\xi(\cdot)$ is a correctness evaluation function, and $f(\cdot)$ is a deterministic language model. Then $
(U_1,J_1),\dots,(U_{n+1},J_{n+1})
$ are exchangeable.
\end{lemma}
Lemma~\ref{lem: uq_exchangeable} (proved in Appendix~\ref{appendix: proof of uq_exchangeable}) allows reasoning about the subset of uncertainty scores corresponding to correctly predicted samples. Specifically, conditional on the event that the test prediction is correct ($J_{n+1}=1$), the test uncertainty score $U_{n+1}$ is exchangeable with the set of calibration uncertainty scores for correctly predicted samples
$
\mathcal{U} \coloneqq \{\, U_i : J_i = 1,\ i = 1,\dots,n \,\}$.

Let $c$ denote the cardinality of $\mathcal{U}$, i.e., $c := |\mathcal{U}|$. The maximal coverage level associated with $\mathcal{U}$ that are no greater than the threshold $\tau$ is given by
\[
1-\beta
\;\coloneqq\;
\sup\{
q \in (0,1) : 
\operatorname{Quantile}(\lceil q(c+1)\rceil/{c}, \mathcal{U}) \le \tau
\}.
\]
Let $\sigma := \operatorname{Quantile}\bigl(\lceil(1-\beta)(c+1)\rceil / c, \mathcal{U}\bigr)$ be the quantile value of $\mathcal{U}$ corresponding to the maximal coverage level $1-\beta$ below $\tau$. By construction, we have $\sigma \le \tau$, and hence
\begin{equation}\label{eq: correct sample retain ratio}
            \mathbb{P}(X_{n+1}\in\mathcal{K}\mid J_{n+1}=1)=\mathbb{P}(U_{n+1}\le \tau \mid J_{n+1}=1)\ge \mathbb{P}(U_{n+1}\le \sigma \mid J_{n+1}=1) \ge 1-\beta,
\end{equation}
providing a lower bound on the probability that a correct prediction is retained (i.e., not abstained).

By exchangeability, the test prediction is correct with probability $\mathbb{P}(J_{n+1}=1) = c/n$.
Finally, combining the lower bound in Eq.~(\ref{eq: correct sample retain ratio}) with the upper bound in Eq.~(\ref{eq: participation guarantee}) via Bayes' rule, the \emph{conditional correctness guarantee} follows
\begin{equation}\label{eq: conditional correctness guarantee}
    \mathbb{P}(J_{n+1}\shorteq{0.1em}1\mid X_{n+1}\in\mathcal{K})\shorteq{0.1em} \frac{\mathbb{P}(X_{n+1}\in\mathcal{K}\mid J_{n+1}\shorteq{0.1em}1)\mathbb{P}(J_{n+1}\shorteq{0.1em}1)}{\mathbb{P}(X_{n+1}\in\mathcal{K})} \ge \frac{1-\beta}{1-\alpha+{1}/(n+1)}\cdot\frac{c}{n},
\end{equation}
which bounds on the probability that a prediction is correct, conditioned on model participation.

\begin{wrapfigure}[10]{r}{8cm}
\centering
\captionsetup{singlelinecheck = false, skip=5pt, justification=justified}
\vspace{-10pt}
  \includegraphics[scale=0.35]{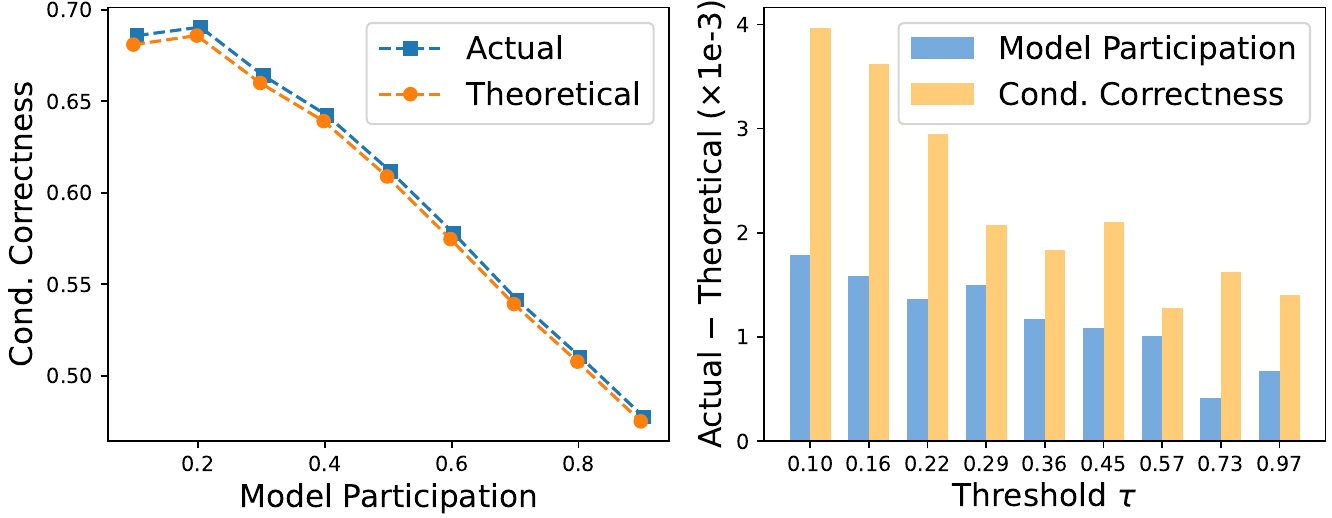}
  \caption{Participation and conditional correctness are ensured.}
  \label{fig: val} 
\end{wrapfigure}
We validate model participation and conditional correctness on Simple-Questions-Wikidata~\cite{wikidata-benchmark} using Gemma-3-4B-Instruct~\cite{gemmateam2025gemma3technicalreport}, with perplexity (Eq.~\ref{eq: perplexity}) as the uncertainty score. As shown in Figure~\ref{fig: val}, both participation rates and conditional correctness consistently exceed their theoretical guarantees across a wide range of thresholds $\tau$, yielding positive gaps.

To achieve high conditional correctness, it is crucial to increase \(1 - \beta\) in Eq.~(\ref{eq: conditional correctness guarantee}), since \(1 - \alpha\) is user-specified and \(c/n\) is determined by the model’s inherent accuracy. As shown in Eq.~(\ref{eq: correct sample retain ratio}), \(1 - \beta\) lower bounds the retention rate of correct predictions. This quantity increases when correct predictions are assigned lower uncertainty scores. Equivalently, improving conditional correctness requires assigning higher confidence to correct predictions, thereby increasing their probability of being retained.
\section{Calibration via Representation Geometry}
We posit that a model’s prediction is more likely to be correct when its internal knowledge actively shapes the representations from which the prediction tokens are decoded. 
Since MLPs serve as the primary locus of memorized knowledge~\cite{dong2025attention}, we analyze whether their updates induce substantive geometric transformations of the representations across layers, as opposed to simple passive propagation. 
These geometry-based signals are subsequently used to calibrate prediction confidence.
\subsection{Preliminary}
Consider a language model composed of $L$ stacked transformer layers. Let a query $x=(x[1],\dots,x[T])$ of $T$ tokens have hidden states $e^l=(e^l[1],\dots,e^l[T])\in\mathbb{R}^{d\times T}$ produced by layer $l$, with $e^0$ as the input embeddings. Each layer updates the representations as 
\begin{equation}
        \tilde{e}^{l}=e^{l-1}+\mathrm{Attn}^l(e^{l-1}),\quad e^l=\tilde{e}^{l}+\mathrm{MLP}^l(\tilde{e}^{l}),
\end{equation}
where $\mathrm{Attn}^l(\cdot)$ and $\mathrm{MLP}^l(\cdot)$ denote the self-attention and MLP transformations at layer $l$, respectively, and $\tilde{e}^l$ is the intermediate representation after self-attention.
However, the changes in representations do not disentangle whether they result from contextual understanding of the query or from knowledge injection, necessitating an explicit assessment of their origin.

\subsection{Knowledge Contribution}
Layer-wise knowledge contribution to representation changes is derived by aggregating information flow~\cite{ferrando2022measuring} between self-attention and MLP blocks.

\paragraph{Information flow within language models}
Assume that a language model uses multi-head self-attention with $H$ heads, each of dimension $D_H=D/H$. For head $h\in\{1,...,H\}$ at layer $l$, the self-attention block produces an attention matrix $A^{lh}\in\mathbb{R}^{T\times T}$, and the corresponding value matrix is denoted $W_V^{lh}\in\mathbb{R}^{D_H\times D}$. Each layer has a projection matrix $W_O^l\in\mathbb{R}^{D\times D}$, which can be partitioned into $H$ submatrices $W_O^{lh}\in\mathbb{R}^{D\times D_H}$. With these notations, we can quantify the component from the source token embedding $e^{l-1}[s]$ to the self-attention output of the target token $\tilde{e}^l[t]$ by
\begin{equation*}
    a(t,s):=\mathbf{1}_{\{s=t\}}e^{l-1}[s]+\sum\nolimits_{h=1}^H {W}_O^{lh}{A}^{lh}[t,s]W_V^{lh}e^{l-1}[s],
\end{equation*}
where $\mathbf{1}_{\{s=t\}}$ is an indicator function that equals 1 if $s=t$ and 0 otherwise. Thereby, the target token embedding after self-attention transformation can be written as $\tilde{e}^{l}[t]=\sum\nolimits_{s=1}^Ta(t,s)$.
\begin{definition}[Normalized Proximity]~\cite{ferrando2022measuring}\label{def: proximity}
Let $z\in\mathbb{R}^D$ admit a decomposition $z=\sum_{i=1}^n z_i$. The normalized proximity of component $z_i$ and $z$ follows
\begin{equation}\label{eq: proximity}
\mathrm{prox}\left(z,z_i\right):=\frac{\max(0,||z||_1-||z-z_i||_1)}{\sum_{j=1}^n\max(0,||z||_1-||z-z_j||_1)}.
\end{equation}
\end{definition}
High proximity indicates a strong contribution of $z_i$ to $z$.
Based on Definition~\ref{def: proximity}, we denote the attention contribution at layer $l$ by $C^l_\mathrm{Attn}$, with elements 
\begin{equation}
    C^l_\mathrm{Attn}[t,s] := \mathrm{prox}(\tilde{e}^l[t], a(t,s)).
\end{equation}
Due to autoregressive attention, each token can only depend on itself and its
predecessors. Hence, attention matrix $A^{lh}$ satisfies $A^{lh}[t,s]=0$ if $s>t$. As a result, $C^l_\mathrm{Attn}$ is lower-triangular with entries
above the diagonal equal to zero.

Since MLP blocks act independently on each token, the hidden state update by MLP can be written as
$
e^l[t]=\tilde{e}^l[t]+m^l[t]$, where $ m^l[t]:=\mathrm{MLP}^l(\tilde{e}^l[t])$.
Using Definition~\ref{def: proximity}, we define the MLP contribution matrix at layer $l$ as $C_{\mathrm{MLP}}^l$. Its diagonal entries are given by
\begin{equation}
    C_{\mathrm{MLP}}^l[t,t] := \mathrm{prox}(e^l[t], m^l[t]),
\end{equation}
while all off-diagonal entries are zero: $    C_{\mathrm{MLP}}^l[t,s] := 0, \quad \forall s \neq t$,
reflecting the absence of cross-token interactions on MLP operations. The information flow and contribution matrices for layer $l$ are visualized in Figure~\ref{fig: matrices and knowledge}(a).
\begin{figure*}[t]
\centering
\captionsetup{singlelinecheck = false, skip=5pt, justification=justified}
  \includegraphics[scale=0.38]{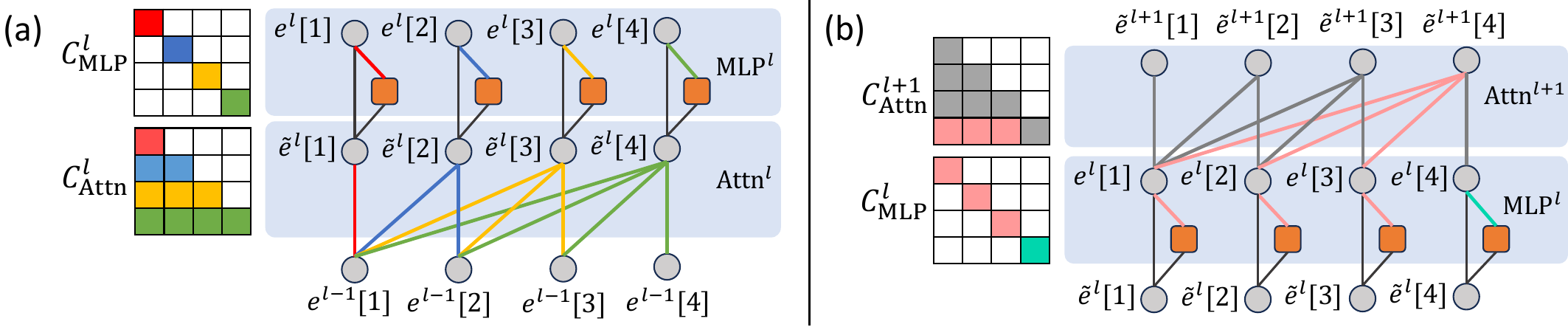}
  \vspace{-1pt}
  \caption{(a)
  Information flow and contribution matrices of layer $l$. Edge colors in the flow graph correspond to the different values in the contribution matrices (no edges means zero contribution). (b) Knowledge contribution from the MLP block at layer $l$. The \textbf{\textcolor{cyan}{cyan edge and entry}} indicate direct contribution to the last input token, corresponding to Eq.~(\ref{eq: direct omega}), while the \textbf{\textcolor{pink}{pink edges and entries}} denote propagated contributions through subsequent self-attention, as described in Eq.~(\ref{eq: prop omega}).}
  \label{fig: matrices and knowledge} 
  \vspace{-10pt}
\end{figure*}
\paragraph{Knowledge Contribution}
We use the matrices $C^l_\mathrm{Attn}$ and $C^l_\mathrm{MLP}$ for $l=1,..., L$ to measure how the knowledge stored in MLP blocks shapes the prediction probability of the generated token at position $T+1$ throughout layers. Figure~\ref{fig: matrices and knowledge} (b) shows that the gap between $\tilde{e}^l[T]$ and $e^l[T]$ is induced by the MLP update to the last query token, $m^l[T]$, so we define the \textcolor{cyan}{direct knowledge contribution} as
\begin{equation}\label{eq: direct omega}
    {\omega^l_\mathrm{dir}}:=C^l_\mathrm{MLP}[T,T].
\end{equation}
The effect of layer $l$'s MLP on earlier tokens is not reflected in $e^l[T]$ but propagates via self-attention to $\tilde{e}^{l+1}[T]$ in the next layer. Thus, the \textcolor{pink}{propagated knowledge contribution} is
\begin{equation}\label{eq: prop omega}
{\omega^l_\mathrm{prop}}:=
\operatorname{diag}\!\big(C^l_{\mathrm{MLP}}\big)_{1:T-1}^{\top}
\, C^{\,l+1}_{\mathrm{Attn}}[T,1:T-1], 
\end{equation}
where $\operatorname{diag}(\cdot)_{1:T-1}$ denotes the vector formed by the first $T\!-\!1$ diagonal entries of its matrix argument. 

Together, Eq.~(\ref{eq: direct omega}) and Eq.~(\ref{eq: prop omega}) capture how MLP-stored knowledge actively shapes the representation of the last query token. We denote the full contribution trajectory by 
\begin{equation}
    \Omega:=(\omega_\mathrm{dir}^1, \omega_\mathrm{prop}^1,...,\omega_\mathrm{dir}^{L-1},\omega_\mathrm{prop}^{L-1},\omega_\mathrm{dir}^L)\in\mathbb{R}^{2L-1}.
\end{equation}
In the final layer, MLP updates to earlier tokens do not influence the last query token through self-attention, so the propagated contribution vanishes, and $\omega_\mathrm{prop}^L$ is not defined.
\subsection{Embedding Rotation}
A high knowledge contribution does not always change the model’s decision, which occurs when the token with the highest predicted probability at one layer is replaced by another token in the next layer. Therefore, we aim to link representation geometry to these decision dynamics.
\begin{lemma}[Scale Invariance of Token Ranking]\label{lem: rescale}
    Let $\varepsilon\in\mathbb{R}^D$ be a token embedding, and let $o:=W_E\cdot\varepsilon\in\mathbb{R}^V$ be the corresponding logits, where $W_E$ is the unembedding matrix. Let $\pi$ denote the ranking such that $o[\pi(1)]\geq o[\pi(2)]\geq\cdots\geq o[\pi(V)]$. For any $\lambda>0$, define $\varepsilon'=\lambda\varepsilon$ and $o'= W_E\cdot\varepsilon'$. Then the ranking is invariant under positive scaling:
    \begin{equation}
        o'[\pi(1)]\geq o'[\pi(2)]\geq\cdots\geq o'[\pi(V)]
    \end{equation}
\end{lemma}
\begin{lemma}[Rotation Sensitivity of Token Ranking]\label{lem: rotation}
Let $\varepsilon,\varepsilon'\in\mathbb{R}^D$ be two unit token embeddings with logits  
$o=W_E\cdot\varepsilon,\; o'=W_E\cdot\varepsilon' \in \mathbb{R}^V$ under the unembedding matrix $W_E$.  
Denote by $v_*=\arg\max_v o[v]$
the predicted token under $o$, and let $\theta$ be the angle between $\varepsilon$ and $\varepsilon'$.  If
\begin{equation}\label{eq: rotation}
    2\|W_E\|_2\,\sin(\theta/2)
\;\le\;
\min_{v\neq v_*}(o[v_*]-o[v]),
\end{equation}
the top token is preserved, i.e., $\arg\max_{v} o'[v] \;=\; v_*$.
\end{lemma}
By Lemma~\ref{lem: rescale} (proved in Appendix~\ref{appendix: proof of rescale}), rescaling an embedding does not affect the model's prediction. In contrast, Lemma~\ref{lem: rotation} (proved in Appendix~\ref{appendix: proof of rotation}) shows that large rotations can violate Eq.~(\ref{eq: rotation}) and change the model’s decision.
Motivated by these results, we track how the embedding of the last query token rotates across layers. Specifically, for layer $l$, let $\theta_\mathrm{dir}^l$ be the angle between $\tilde{e}^l[T]$ and $e^l[T]$, and let $\theta_\mathrm{prop}^l$ be the angle between $e^l[T]$ and $\tilde{e}^{\,l+1}[T]$. We define the rotation trajectory by
\begin{equation}\label{eq: Theta}
    \Theta:=(\theta_\mathrm{dir}^1, \theta_\mathrm{prop}^1,...,\theta_\mathrm{dir}^{L-1},\theta_\mathrm{prop}^{L-1},\theta_\mathrm{dir}^L)\in\mathbb{R}^{2L-1}.
\end{equation}
\begin{wrapfigure}[12]{r}{5.5cm}
\centering
\captionsetup{singlelinecheck = false, skip=5pt, justification=justified}
  \includegraphics[scale=0.28]{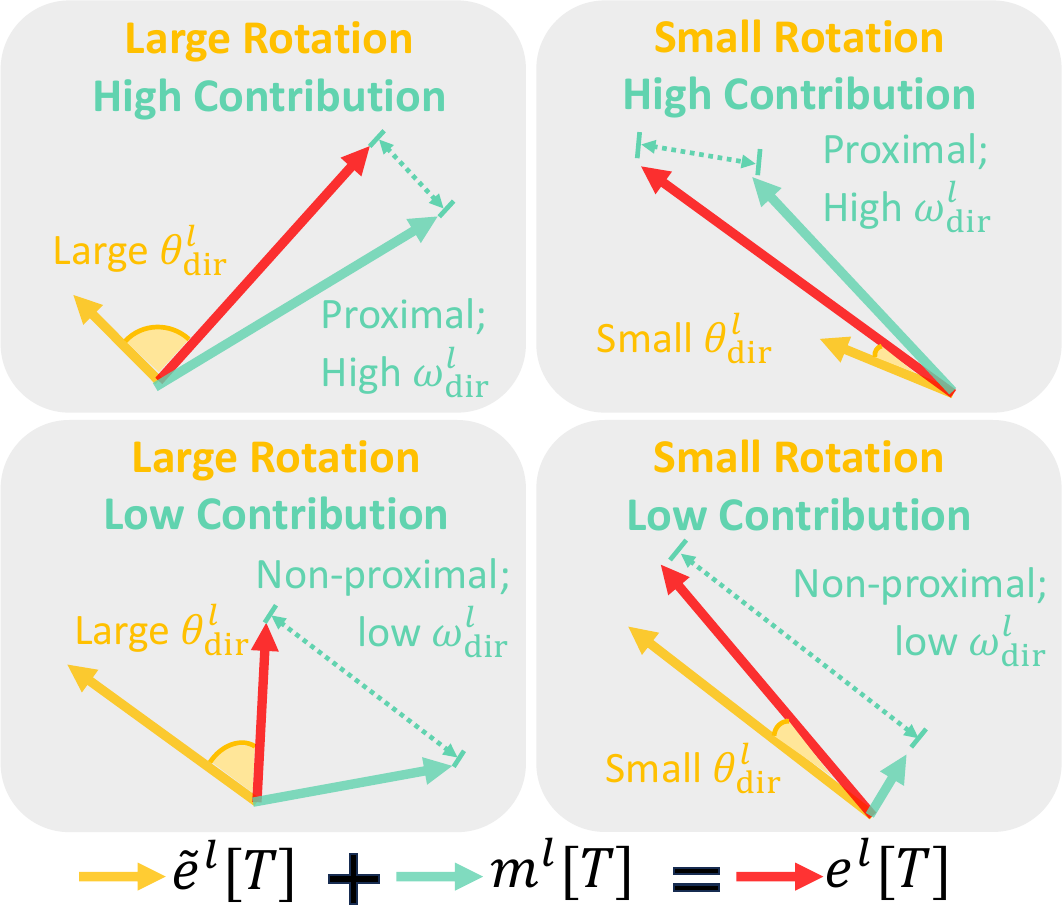}
  \caption{The relation between direct knowledge contribution and embedding rotation.}
  \label{fig: rotation} 
\end{wrapfigure}
The relation between \textcolor{cyan}{direct knowledge contribution $\omega_\mathrm{dir}^l$} and its corresponding \textcolor{gold}{embedding rotation $\theta^l_\mathrm{dir}$} is shown in Figure~\ref{fig: rotation}, which can be similarly extended to propagated contributions. 
Overall, knowledge contribution and embedding rotation are not necessarily tightly coupled, as similar levels of contribution can lead to different degrees of rotational change depending on the local geometry.

\textbf{Hypothesis 1.}
Correctness is associated with cases where strong knowledge contribution induces significant embedding rotation, reflecting an effective redirection of the representation with increased prediction confidence.

\subsection{Anisotropy}
Even with high knowledge contribution and large embedding rotation, the decision can still be incorrect due to misleading knowledge. Therefore, it is crucial to evaluate the likelihood that the resulting prediction is truly correct.

Transformer-based language models exhibit embeddings of training or in-distribution samples that concentrate within a narrow cone in the hidden space $\mathbb{R}^D$, inducing strong anisotropy~\cite{ethayarajh2019contextual}.
This concentration yields well-defined dominant directions at each layer, which serve as canonical directions for typical in-distribution representations.

To distinguish between in-distribution and out-of-distribution behaviors, we define separate embedding distributions. Specifically, for layer $l$, let 
$\tilde{\eta}^{l,{\mathrm{in}}}$ and $\eta^{l,{\mathrm{in}}}$
denote the expectation of token embeddings produced by self-attention and MLP blocks for in-distribution inputs, and $\tilde{\eta}^{l,{\mathrm{out}}}$ and $\eta^{l,{\mathrm{out}}}$ denote their out-of-distribution counterparts. 

\textbf{Hypothesis 2.}
In-distribution queries induce hidden states that closely align with $\tilde{\eta}^{l,\mathrm{in}}$ and $\eta^{l,\mathrm{in}}$, remaining within the anisotropic cone and leading to more reliable predictions. In contrast, out-of-distribution queries deviate from this cone and shift toward $\tilde{\eta}^{l,\mathrm{out}}$ and $\eta^{l,\mathrm{out}}$. Such deviations indicate extrapolation beyond the model’s training-induced representation region, resulting in uncertainty~\cite{ren2023outofdistributiondetectionselectivegeneration}.

\begin{wrapfigure}[14]{r}{5.55cm}
\centering
\captionsetup{singlelinecheck = false, skip=5pt, justification=justified}
\vspace{-8pt}
  \includegraphics[scale=0.28]{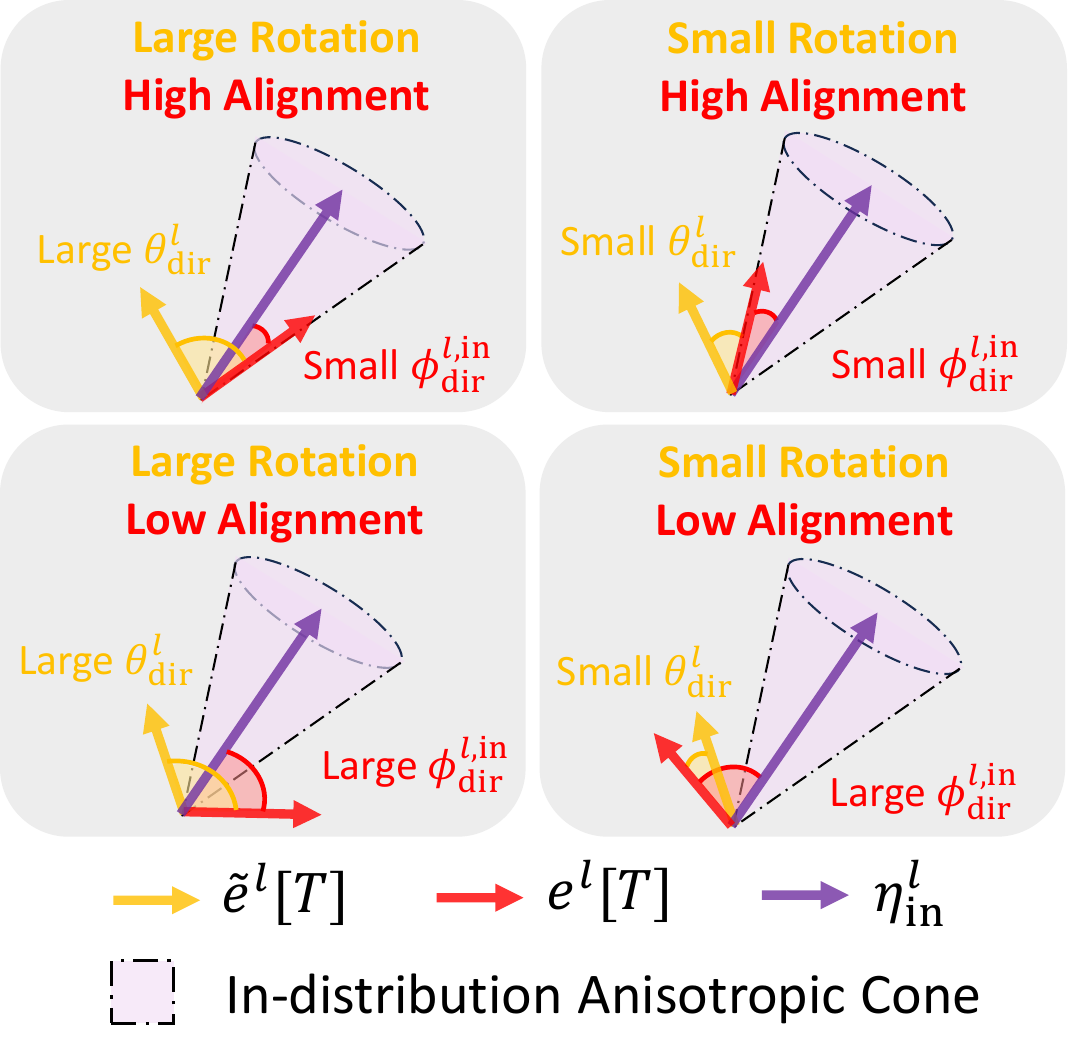}
  \vspace{-7pt}
  \caption{Dynamic rotation and static alignment are complementary but independent characterizations of prediction reliability.}
  \label{fig: anisotropy} 
\end{wrapfigure}
Analogous to Eq.~(\ref{eq: Theta}), for alignment to dominant in-distribution directions, we define $\phi^{l,\mathrm{in}}_{\mathrm{dir}}$ as the angle between $\eta^l{\mathrm{in}}$ and $e^l[T]$, and $\phi^{l,\mathrm{in}}_{\mathrm{prop}}$ as the angle between $\tilde{\eta}^{l+1}_{\mathrm{in}}$ and $\tilde{e}^{l+1}[T]$. Similarly, we define $\phi^{l,\mathrm{out}}_{\mathrm{dir}}$ and $\phi^{l,\mathrm{out}}_{\mathrm{prop}}$ with respect to $\eta^{l,{\mathrm{out}}}$ and $\tilde{\eta}^{l+1}_{\mathrm{out}}$. The anisotropy trajectories are given by
\begin{equation}\label{eq: Phi}
\begin{split}
    &\Phi^{\mathrm{in}} := (\phi_{\mathrm{dir}}^{1,\mathrm{in}}, \phi_{\mathrm{prop}}^{1,\mathrm{in}}, \ldots, \phi_{\mathrm{dir}}^{L,\mathrm{in}}),
\\&\Phi^{\mathrm{out}} := (\phi_{\mathrm{dir}}^{1,\mathrm{out}}, \phi_{\mathrm{prop}}^{1,\mathrm{out}}, \ldots, \phi_{\mathrm{dir}}^{L,\mathrm{out}}).
\end{split}
\end{equation}
Crucially, this \textit{static} deviation from anisotropy in Eq.~(\ref{eq: Phi}) is complementary to the \textit{dynamic} rotation measures between layers in Eq.~(\ref{eq: Theta}). We employ \textcolor{gold}{embedding rotation $\theta^l_{\mathrm{dir}}$} and the in-distribution \textcolor{red}{alignment $\phi^{l,\mathrm{in}}_{\mathrm{dir}}$} from the direct contribution case as illustrative examples to visualize their relation in Figure~\ref{fig: anisotropy}.
\subsection{Token-level Confidence Calibration}
Previous analyses focused on the first generated token. 
We now extend this framework to all generated tokens, assigning each token a feature vector 
$(\Omega, \Theta, \Phi^\mathrm{in}, \Phi^{\mathrm{out}}) \in \mathbb{R}^{8L-4}$.

To assess whether the proposed token-level features reflect prediction correctness, we employ the Mahalanobis distance~\cite{lee2018simpleunifiedframeworkdetecting}. Specifically, we collect token-level feature vectors from correct predictions and compute their empirical mean and covariance, denoted by $\mu_{\mathrm{corr}}$ and $\Sigma_{\mathrm{corr}}$. Given a test feature vector $\nu$, its Mahalanobis distance to this distribution is
$d_\mathrm{corr} = \sqrt{ (\nu - \mu_\mathrm{corr})^\top \Sigma_\mathrm{corr}^{-1} (\nu - \mu_\mathrm{corr}) }$.
Analogously, we define $d_\mathrm{inc}$ as the distance to the feature distribution of incorrect predictions.

\begin{wrapfigure}[14]{r}{5.55cm}
\centering
\captionsetup{singlelinecheck = false, skip=5pt, justification=justified}
  \includegraphics[scale=0.35]{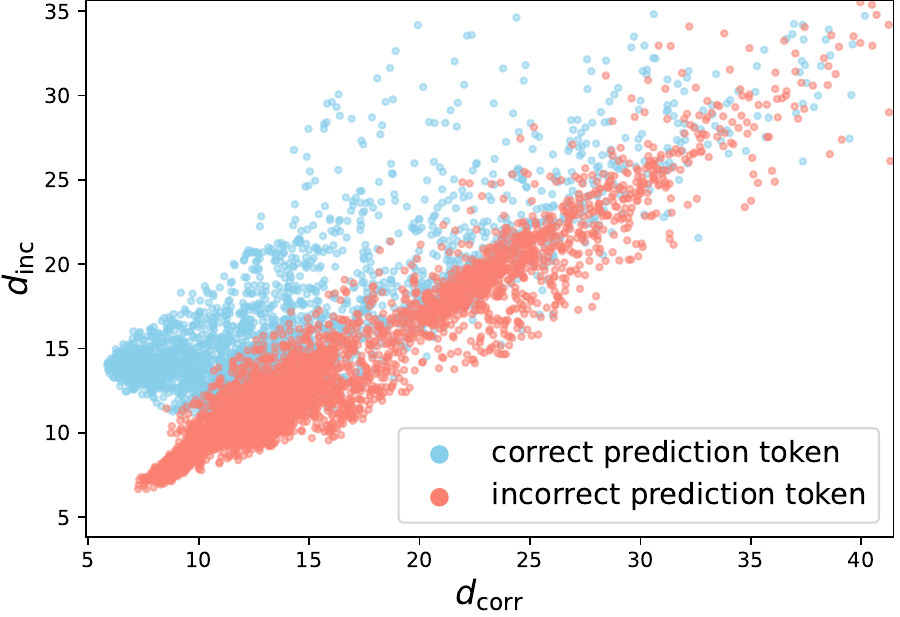}
  \caption{Separation of generated tokens from correct and incorrect predictions using $d_{\mathrm{corr}}$ and $d_{\mathrm{inc}}$.}
  \label{fig: separate} 
\end{wrapfigure}
Tokens from correct predictions should exhibit small $d_{\mathrm{corr}}$ and large $d_{\mathrm{inc}}$, while the opposite is expected for incorrect predictions. We empirically evaluate this separability using Gemma-3-4B-Instruct on Simple-Questions-Wikidata. A reference set is used to estimate $\mu_{\mathrm{corr}}, \mu_{\mathrm{inc}}$ and $\Sigma_{\mathrm{corr}}, \Sigma_{\mathrm{inc}}$.
Similarly, we treat correctly predicted samples in the reference set as in-distribution inputs to compute $\tilde{\eta}^{l,{\mathrm{in}}}$ and $\eta^{l,{\mathrm{in}}}$, while incorrectly predicted samples are treated as out-of-distribution inputs to compute $\tilde{\eta}^{l,{\mathrm{out}}}$ and $\eta^{l,{\mathrm{out}}}$. We combine the token-level features $d_\mathrm{corr}$ and $d_\mathrm{inc}$ with the uncalibrated confidence to train a calibrator, and aggregate the calibrated confidences to produce response-level uncertainty scores via perplexity in Eq.~(\ref{eq: perplexity}).

\section{Experiment}
\textbf{Dataset.} We conduct experiments on six datasets: Natural Questions (NQ)~\cite{kwiatkowski2019natural} and TruthfulQA (TQA)~\cite{lin2022truthfulqameasuringmodelsmimic} for open-ended question answering; Simple Questions Wiki (SQW)~\cite{wikidata-benchmark} and SciQ~\cite{welbl2017crowdsourcingmultiplechoicescience} for closed-form question answering; GSM8K~\cite{cobbe2021gsm8k} and CommonsenseQA (CQA)~\cite{talmor2019commonsenseqaquestionansweringchallenge} for reasoning. Each dataset is partitioned into reference, training, calibration, and test splits with a 5:3:1:1 ratio.

\textbf{Model Selection.} We use Gemma-3-4B-Instruct~\cite{gemmateam2025gemma3technicalreport}, LLaMA-3.2-3B-Instruct, and LLaMA-3-8B-Instruct~\cite{grattafiori2024llama3herdmodels} as the base inference models. HHEM-2.1-Open~\citep{hhem-2.1-open}, GPT-4o~\cite{openai2024gpt4ocard}, and GPT-5.1~\cite{openai_gpt5} are jointly used to define the evaluation function $\xi$, as detailed in Appendix~\ref{appendix: correctness eval}.
The proposed calibrator is trained using the XGBoost library~\cite{Chen_2016}.

\textbf{Baselines.} The uncertainty scoring function $r$ can be any uncertainty quantification (UQ) method that outputs a scalar reliability estimate. We compare against representative baselines from the widely used llm-polygraph benchmark~\cite{shelmanovvashurin2025}, which span several categories. Likelihood-based methods, such as naïve perplexity~\cite{fomicheva2020unsupervised} and SAR~\cite{duan2024shiftingattentionrelevancepredictive}, rely on token-level probabilities and entropy of the predictive distribution. Consistency-based approaches estimate uncertainty via response diversity, e.g., Semantic Entropy~\cite{kuhn2023semantic}. Attention-based methods exploit internal attention patterns, including Focus~\cite{zhang-etal-2023-enhancing-uncertainty} and Attention Score~\cite{chen2024insidellmsinternalstates}. Embedding-based methods use hidden representations, such as Eigen Score~\cite{sriramanan2024llmcheck} and ATRMD~\cite{vazhentsev2025tokenleveldensitybaseduncertaintyquantification}. Finally, verbalized methods are represented by P(True)~\cite{kadavath2022language}, which directly elicits self-reported confidence from the model.
The conformal abstention procedure for each method is performed over 1000 trials, with random resplitting of the calibration and test sets.
\subsection{Hypothesis Validation}
\begin{wrapfigure}[12]{r}{4.4cm}
\centering
\captionsetup{singlelinecheck = false, skip=5pt, justification=justified}
\vspace{-15pt}
  \includegraphics[scale=0.35]{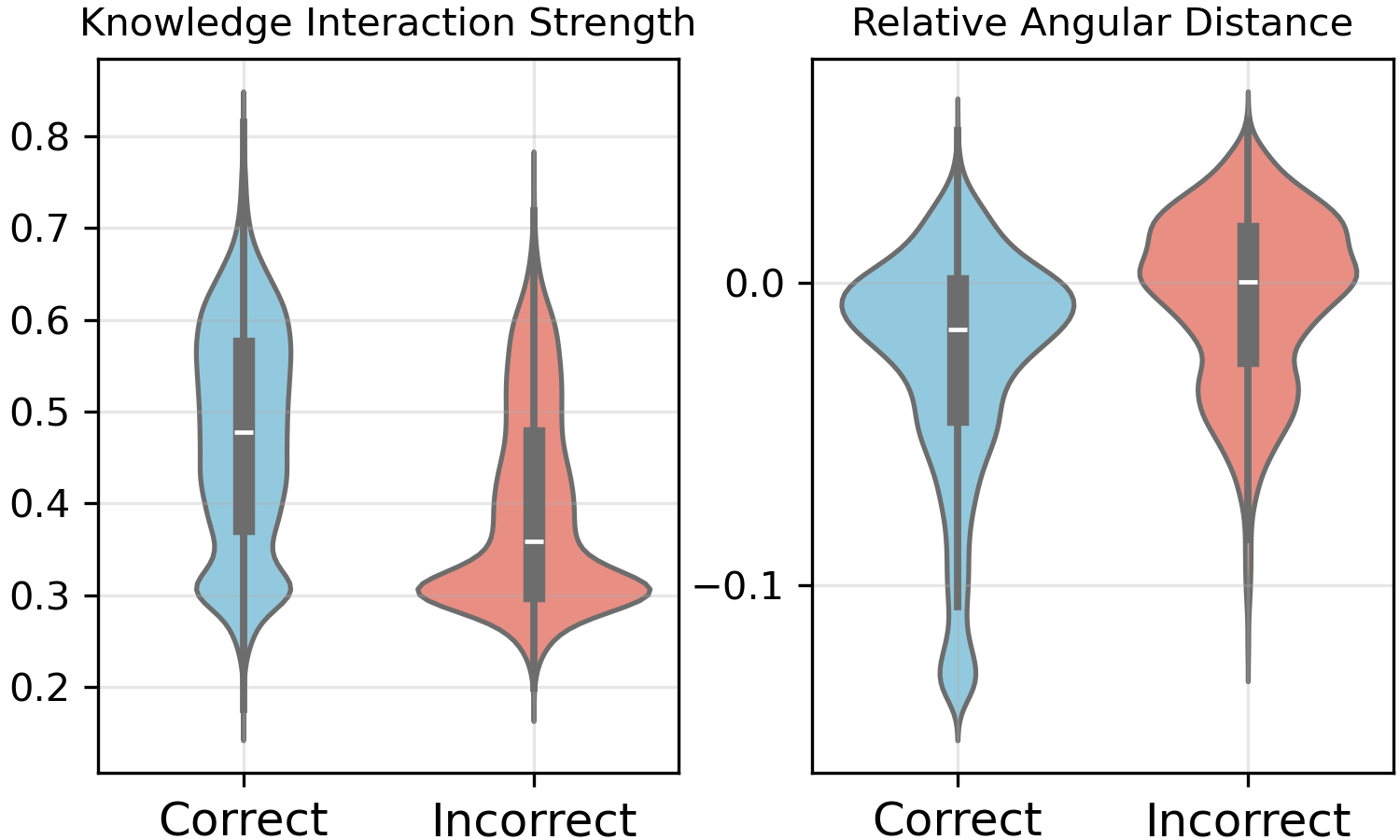}
  \caption{Knowledge interaction strength and relative angular distance between correct and incorrect predictions.}
  \label{fig: violin} 
\end{wrapfigure}
We validate Hypotheses~1 and~2 on reference sets using correct and incorrect predictions, with results averaged across datasets. For each sample, we compute the knowledge interaction strength and relative angular distance by layer-wise averaging $\Omega\odot\Theta\in\mathbb{R}^{2L-1}$ and $\Phi^\mathrm{in}-\Phi^\mathrm{out}\in\mathbb{R}^{2L-1}$, respectively, and compare the aggregated values between correct and incorrect predictions. As shown in Figure~\ref{fig: violin} (left), correct predictions exhibit stronger knowledge interaction, with higher contribution and larger rotation than incorrect ones, supporting Hypothesis~1. Figure~\ref{fig: violin} (right) shows that correct predictions are closer to in-distribution representations, reflected by smaller relative angular distance, validating Hypothesis~2.

\subsection{Main Results}
The main experimental results using Gemma-3-4B-Instruct as the inference model are shown in Figure~\ref{fig: gemma_auc}. Across various model participation rate \(1-\alpha\) ranging from 0.1 to 0.9, the proposed method achieves higher conditional correctness, which is bounded in Eq.~(\ref{eq: conditional correctness guarantee}). This indicates that incorrect predictions receive higher uncertainty scores and are more likely to be abstained from. As \(1-\alpha\) increases, the results of all methods tend to converge, which is expected since abstaining from fewer samples causes conditional correctness to approach the overall accuracy \(c/n\).
\begin{figure*}[t]
\vspace{-2pt}
\centering
\captionsetup{singlelinecheck = false, skip=5pt, justification=justified}
  \includegraphics[scale=0.5]{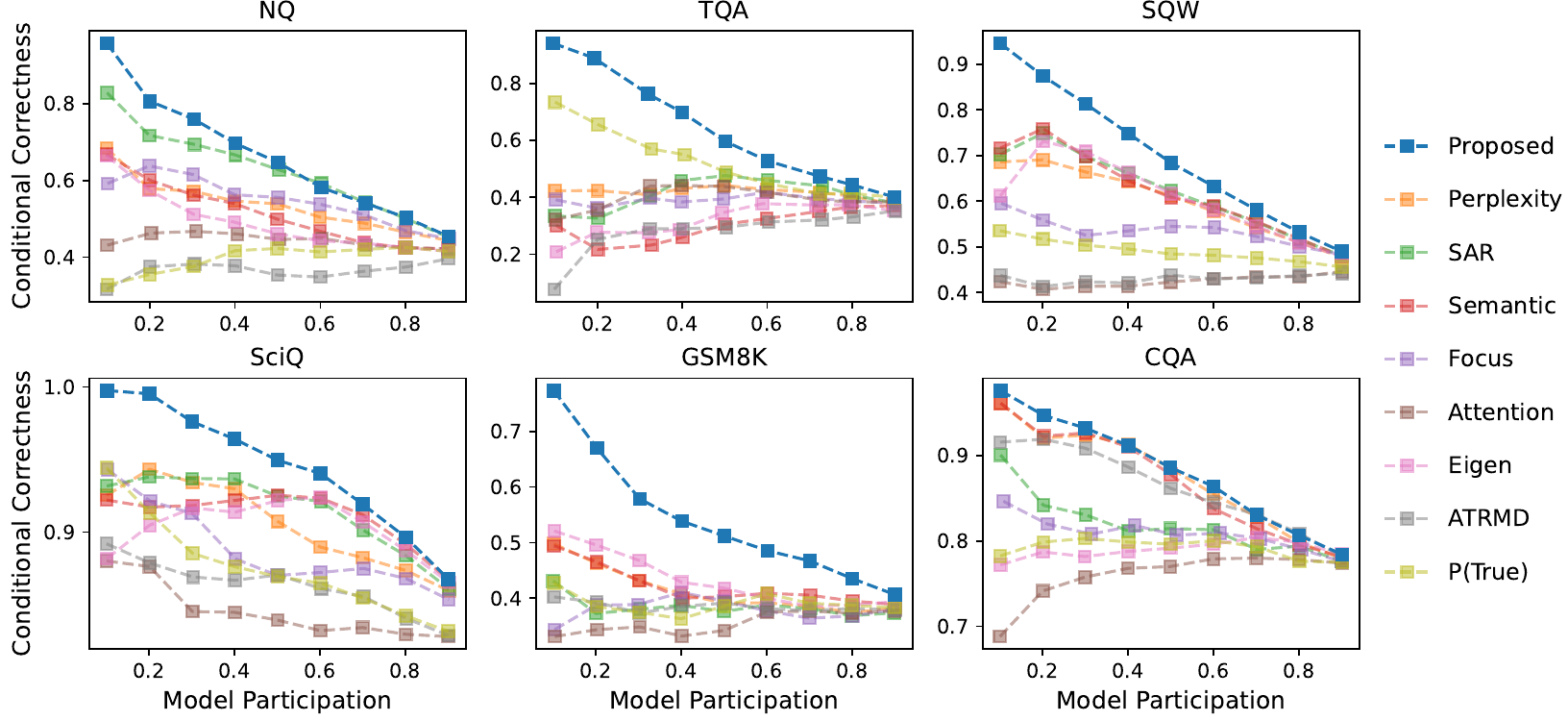}
  \caption{Conditional correctness using different UQ methods as uncertainty scoring functions, with Gemma-3-4B-Instruct as the inference model. We achieve the highest conditional correctness across participation rates.}  \label{fig: gemma_auc} 
\end{figure*}

In addition, we report comparisons outside the conformal
abstention framework using AUROC and AUPRC. Consistent with the trends observed in Figure~\ref{fig: gemma_auc}, the proposed method consistently outperforms baselines in all datasets. Experimental results using LLaMA-3.2-3B-Instruct and LLaMA-3-8B-Instruct as the inference models are presented in Appendix~\ref{appendix: additional experiments}. Furthermore, as our method involves training a calibrator, we report the performance of \emph{calibrated} baselines in the Appendix for a fair comparison.
\begin{table*}[t]
\centering
\small
\captionsetup{justification=justified}
\caption{AUROC and AUPRC of different uncertainty scores with Gemma-3-4B-Instruct. For each dataset, the highest value is shown in \textbf{bold}.}
\vspace{-5pt}
\setlength{\tabcolsep}{2pt}
\begin{tabularx}{\textwidth}{c|YYYYYY|YYYYYY} 
\specialrule{0.4pt}{0pt}{2pt}
\specialrule{0.4pt}{0pt}{3pt}
\multirow{2}{*}{Method} & \multicolumn{6}{c|}{AUROC} & \multicolumn{6}{c}{AUPRC} \\
& NQ & TQA & SQW & SciQ & GSM8K & CQA & NQ & TQA & SQW & SciQ & GSM8K & CQA \\
\midrule

Perplexity      & 0.68 & 0.60 & 0.71 & 0.70 & 0.53 & 0.71 & 0.58 & 0.44 & 0.63 & 0.91 & 0.43 & 0.89 \\
SAR             & 0.80 & 0.62 & 0.73 & 0.74 & 0.51 & 0.60 & 0.71 & 0.44 & 0.65 & 0.91 & 0.40 & 0.83 \\
Semantic        & 0.62 & 0.44 & 0.73 & 0.75 & 0.55 & 0.70 & 0.54 & 0.34 & 0.66 & 0.91 & 0.43 &  0.88\\
Focus           & 0.69 & 0.57 & 0.60 & 0.65 & 0.49 & 0.60 & 0.56 & 0.43 & 0.55 & 0.89 & 0.40 & 0.82 \\
Attention       & 0.56 & 0.58 & 0.47 & 0.53 & 0.47 & 0.52 & 0.46 & 0.42 & 0.44 & 0.85 & 0.35 & 0.76 \\
Eigen           & 0.59 & 0.48 & 0.72 & 0.73 & 0.56 & 0.58 & 0.52 & 0.34 & 0.62 & 0.90 & 0.44 & 0.79 \\
ATRMD         & 0.43 & 0.41 & 0.47 & 0.58 & 0.51 & 0.69 & 0.37 & 0.31 & 0.43 & 0.87 & 0.38 & 0.87 \\
P(True)         & 0.50 & 0.70 & 0.56 & 0.60 & 0.52 & 0.57 & 0.40 & 0.62 & 0.50 & 0.88 & 0.39 & 0.80 \\
Proposed        & \textbf{0.83} & \textbf{0.86} & \textbf{0.82} & \textbf{0.81} & \textbf{0.71} & \textbf{0.72} & \textbf{0.78} & \textbf{0.82} & \textbf{0.81} & \textbf{0.95} & \textbf{0.60} & \textbf{0.90} \\
\bottomrule
\end{tabularx}
\label{table: gemma auroc auprc}
\vspace{-10pt}
\end{table*}
\subsection{Ablation Study on Geometry Signals}
\begin{wrapfigure}[8]{r}{4.7cm}
\centering
\captionsetup{singlelinecheck = false, skip=5pt, justification=justified}
\vspace{-25pt}
  \includegraphics[scale=0.3]{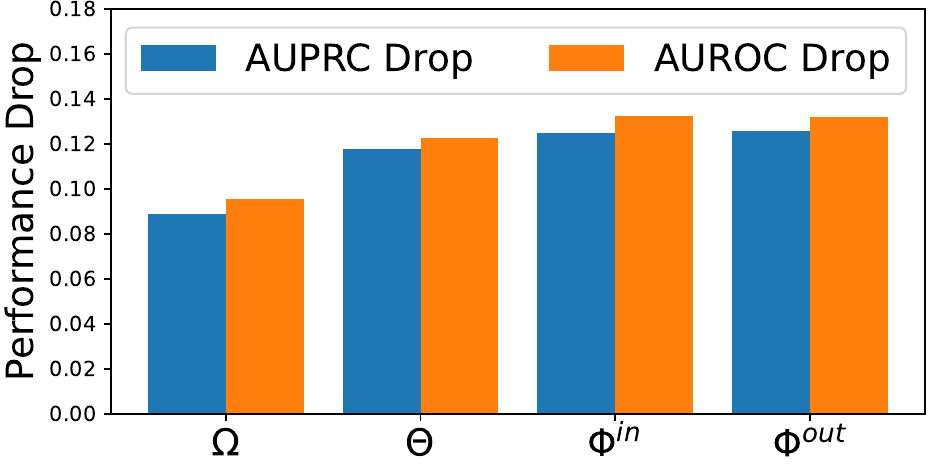}
  \vspace{-15pt}
  \caption{Impact of signal permutation on AUROC and AUPRC.}
  \label{fig: importance} 
\end{wrapfigure}
We conduct an importance analysis of the proposed geometry signals $(\Omega, \Theta, \Phi^{\mathrm{in}}, \Phi^{\mathrm{out}})$ by randomly permuting each component during inference and measuring the resulting drop in AUROC and AUPRC. We report the average performance across all inference models and datasets. Figure~\ref{fig: importance} indicates that perturbing each component leads to comparable degradation, suggesting that they contribute similarly to the overall effectiveness.
\section{Conclusion}
In this work, we introduce Conformal Abstention (CA), which enables language models to selectively abstain when their predictions are unreliable, without requiring additional training. We further developed a representation-geometry-based calibration approach that enhances the alignment between prediction confidence and correctness by quantifying how strongly model knowledge contributes to responses, thereby making abstention decisions more effective at filtering out incorrect predictions.
\newpage
\bibliographystyle{unsrt}
\bibliography{neurips_2026}


\newpage
\appendix

\section{Additional Theoretical Statement}
\subsection{Proof of Lemma~\ref{lem: uq_exchangeable}}\label{appendix: proof of uq_exchangeable}
\begin{proof}
Consider the deterministic mapping
\[
g: ((X_1,Y_1),\dots,(X_{n+1},Y_{n+1})) \;\mapsto\; ((U_1,J_1),\dots,(U_{n+1},J_{n+1})),
\]
where for each $i$, 
\[
U_i = r(f(X_i)\mid X_i), \quad J_i = \xi(X_i,Y_i,f(X_i)),
\] 
and $g$ is the composition of the language model $f$, the uncertainty function $r$, and the correctness evaluation $\xi$.  

Since $(X_1,Y_1),\dots,(X_{n+1},Y_{n+1})$ are exchangeable, for any permutation $\pi$ of $\{1,\dots,n+1\}$,
\[
(X_{\pi(1)},Y_{\pi(1)}),\dots,(X_{\pi(n+1)},Y_{\pi(n+1)}) \stackrel{d}{=} (X_1,Y_1),\dots,(X_{n+1},Y_{n+1}),
\]
where $\stackrel{d}{=}$ indicates distributional equality.
Now consider applying $g$ to each of these pairs. Because $g$ acts independently and deterministically on each input pair, we have
\[
(U_{\pi(1)},J_{\pi(1)}),\dots,(U_{\pi(n+1)},J_{\pi(n+1)}) 
= g((X_{\pi(1)},Y_{\pi(1)}),\dots,(X_{\pi(n+1)},Y_{\pi(n+1)})).
\]

By Theorem~\ref{thm:exchangeability preservation}, if for every permutation $\pi_1$ of the outputs there exists a permutation $\pi_2$ of the inputs such that
\[
\pi_1 g((x_1,y_1),\dots,(x_{n+1},y_{n+1})) = g(\pi_2(x_1,y_1),\dots,\pi_2(x_{n+1},y_{n+1})),
\]
then the output sequence is exchangeable. In our case, we can choose $\pi_2 = \pi_1$ because $g$ is applied independently and deterministically to each input. This satisfies the requirement of the theorem for all possible permutations of outputs.

Therefore, the sequence of uncertainty-correctness pairs $(U_1,J_1),\dots,(U_{n+1},J_{n+1})$ is exchangeable.

\medskip
\noindent
\textbf{Remark:} This result holds for any deterministic uncertainty function $r$, correctness evaluation $\xi$, and natrual language generation model $f$, provided the underlying prompts and ground truths $(X_i,Y_i)$ are exchangeable. The key intuition is that any permutation of the inputs leads to the same permutation in the outputs, preserving exchangeability under deterministic transformations.
\end{proof}
\subsection{Proof of Lemma~\ref{lem: rescale}}\label{appendix: proof of rescale}
\begin{proof}
By definition, the logits are computed as \(o = W_E \cdot \varepsilon\), and the ranking \(\pi\) is determined by the ordering of the entries:
\[
o[\pi(1)] \ge o[\pi(2)] \ge \cdots \ge o[\pi(V)].
\]

Now consider a positive scaling of the embedding, \(\varepsilon' = \lambda \varepsilon\) with \(\lambda > 0\). Then the new logits are
\[
o' = W_E \cdot \varepsilon' = W_E \cdot (\lambda \varepsilon) = \lambda (W_E \cdot \varepsilon) = \lambda o.
\]

Since \(\lambda > 0\), scaling preserves the ordering of the elements of \(o\): for any indices \(i, j\),
\[
o[i] \ge o[j] \implies \lambda o[i] \ge \lambda o[j].
\]

Therefore, the ranking induced by \(o'\) is identical to that induced by \(o\):
\[
o'[\pi(1)] \ge o'[\pi(2)] \ge \cdots \ge o'[\pi(V)].
\]

This completes the proof.
\end{proof}
\subsection{Proof of Lemma~\ref{lem: rotation}}\label{appendix: proof of rotation}
\begin{proof}
We assume the unit token embeddings $\varepsilon, \varepsilon' \in \mathbb{R}^D$ with $\|\varepsilon\|_2 = \|\varepsilon'\|_2 = 1$, and let $\theta$ be the angle between them:
\begin{equation}
    \varepsilon^\top \varepsilon' = \cos\theta.
\end{equation}  

Because $\varepsilon$ and $\varepsilon'$ are unit vectors, we have
\begin{equation}
\begin{split}
\|\varepsilon - \varepsilon'\|_2^2 &= \|\varepsilon\|_2^2 + \|\varepsilon'\|_2^2 - 2 \varepsilon^\top \varepsilon' \\
&= 1 + 1 - 2 \cos\theta = 2(1 - \cos\theta).
\end{split}
\end{equation}

Using the identity $1 - \cos\theta = 2 \sin^2(\theta/2)$, it follows that
\begin{equation}\label{eq:sin}
\|\varepsilon - \varepsilon'\|_2 = 2 \sin(\theta/2).
\end{equation}

Let $o = W_E \cdot\varepsilon$ and $o' = W_E\cdot \varepsilon'$ be the corresponding logits. Then
\begin{equation}
o - o' = W_E (\varepsilon - \varepsilon').
\end{equation}

Applying the spectral norm of $W_E$, we obtain
\begin{equation}
\|o - o'\|_2 \le \|W_E\|_2 \, \|\varepsilon - \varepsilon'\|_2,
\end{equation}
and combining with Eq.~\eqref{eq:sin} gives
\begin{equation}\label{eq:bound_by_sin}
\|o - o'\|_2 \le 2 \|W_E\|_2 \, \sin(\theta/2).
\end{equation}

Define $\gamma := o - o' \in \mathbb{R}^V$ as the change in logits. Let $v_* = \arg\max_v o[v]$ denote the predicted token under $o$. A replacement, i.e., $\arg\max_v o'[v] \neq v_*$, occurs if and only if there exists $v \neq v_*$ such that
\begin{equation}
o'[v] \ge o'[v_*] \iff \gamma[v] - \gamma[v_*] \ge o[v_*] - o[v].
\end{equation}

For any $v \neq v_*$, we have
\begin{equation}
|\gamma[v] - \gamma[v_*]| \le |\gamma[v]| + |\gamma[v_*]| \le 2 \|\gamma\|_\infty.
\end{equation}

Therefore, if
\begin{equation}
2 \|\gamma\|_\infty \le \min_{v \neq v_*} (o[v_*] - o[v]),
\end{equation}
then $\gamma[v] - \gamma[v_*] \le o[v_*] - o[v]$ for all $v \neq v_*$, ensuring no replacement occurs.  

Since $\|\gamma\|_\infty \le \|\gamma\|_2$ and by Eq.~\eqref{eq:bound_by_sin} we have $\|\gamma\|_2 \le 2 \|W_E\|_2 \, \sin(\theta/2)$, a sufficient condition to avoid replacement is
\begin{equation}
2 \|W_E\|_2 \, \sin(\theta/2) \le \min_{v \neq v_*} (o[v_*] - o[v]).
\end{equation}

This completes the proof.
\end{proof} 
\section{Correctness Evaluation}~\label{appendix: correctness eval}
To determine whether a generated response is correct, we move beyond token-level overlap metrics (e.g., BERTScore~\citep{zhang2020bertscoreevaluatingtextgeneration}) and instead adopt a systematic, multi-stage semantic evaluation pipeline leveraging HHEM-2.1-Open~\citep{hhem-2.1-open}, GPT-4o~\cite{openai2024gpt4ocard}, and GPT-5.1~\cite{openai_gpt5}.

First, we convert both the model prediction and the ground-truth answer into natural language statements conditioned on the original question, using Qwen2.5-7B~\citep{qwen2025qwen25technicalreport}. For example, given the question “Who wrote Pride and Prejudice?”, a model prediction such as “It was written by Jane Austen” is reformulated as “The author of Pride and Prejudice is Jane Austen.” The ground-truth answer is transformed into a reference statement with the same structure to ensure comparability. We then compute the semantic similarity between the predicted and reference statements using HHEM-2.1-Open, and classify predictions with similarity scores above 0.8 as correct. This formulation evaluates meaning-level agreement rather than surface-level token overlap, which is essential for recognizing semantically equivalent but lexically diverse answers~\cite{qiu2024semanticdensityuncertaintyquantification}.

However, hallucination detection models such as HHEM-2.1-Open evaluate consistency with respect to a reference rather than absolute correctness, and can therefore misclassify semantically valid responses as incorrect when they are not explicitly supported by the reference~\cite{zhang2026hallucinationdetectionevaluationlarge}. To mitigate this issue, we introduce a second-stage verification step. Predictions initially labeled as incorrect are re-evaluated by GPT-4o and GPT-5.1, which act as independent semantic judges. If both models agree that a response is correct, we reassign the sample to the correct class. This conservative consensus mechanism reduces false negatives while maintaining high precision, ensuring that semantically valid answers are not excluded due to limitations of a single evaluator.

\section{Additional Experiment Results}\label{appendix: additional experiments}
Figures~\ref{fig: llama3B_auc} and~\ref{fig: llama8B_auc} present the conditional correctness–participation trade-off using LLaMA-3.2-3B-Instruct and LLaMA-3-8B-Instruct as the base inference models, respectively. Across participation levels ranging from $1-\alpha=0.1$ to $0.9$, our method consistently achieves higher conditional correctness, indicating that the proposed uncertainty scores more effectively distinguish correct from incorrect responses. Consequently, the abstention mechanism selectively filters out less reliable predictions, resulting in improved accuracy among the retained outputs.

\begin{figure*}[h]
\vspace{-2pt}
\centering
\captionsetup{singlelinecheck = false, skip=5pt, justification=justified}
  \includegraphics[scale=0.5]{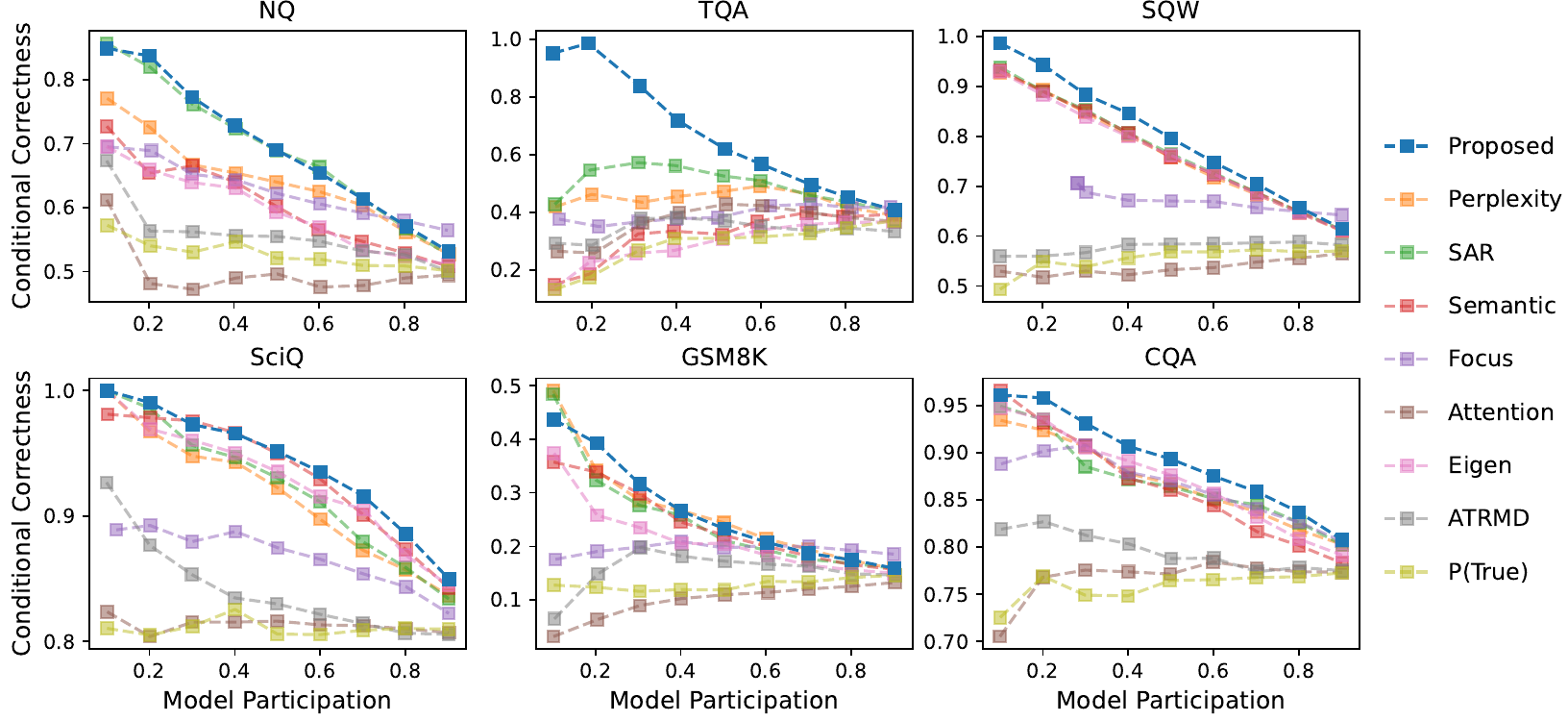}
  \caption{Conditional correctness using different UQ methods as uncertainty scoring functions, with LLaMA-3.2-3B-Instruct as the inference model. We achieve the highest conditional correctness across participation rates.}  \label{fig: llama3B_auc} 
\end{figure*}

\begin{figure*}[h]
\vspace{-2pt}
\centering
\captionsetup{singlelinecheck = false, skip=5pt, justification=justified}
  \includegraphics[scale=0.5]{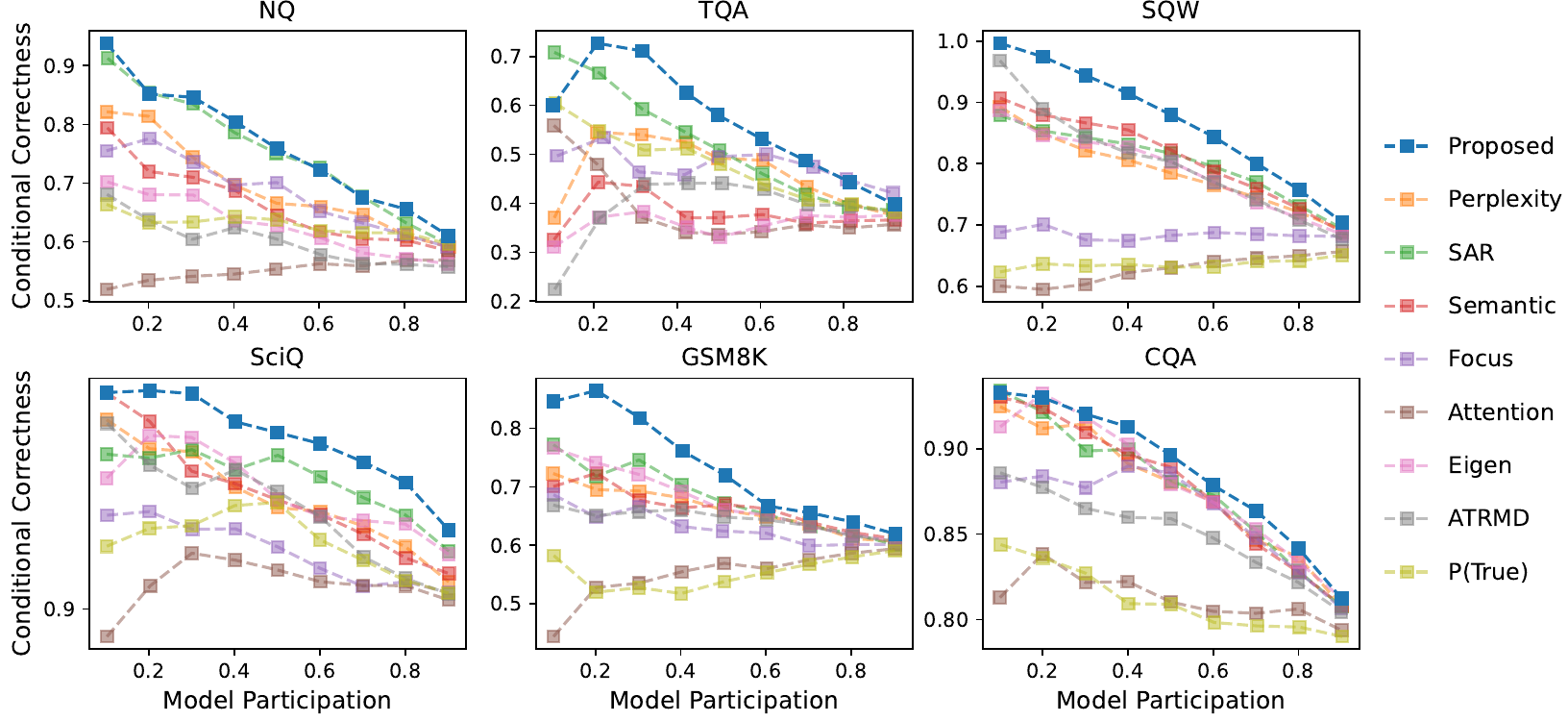}
  \caption{Conditional correctness using different UQ methods as uncertainty scoring functions, with LLaMA-3-8B-Instruct as the inference model. We achieve the highest conditional correctness across participation rates.}  \label{fig: llama8B_auc} 
\end{figure*}

The experimental results for AUROC and AUPRC using LLaMA-3.2-3B-Instruct and LLaMA-3-8B-Instruct are reported in Table~\ref{table: llama3B auroc auprc} and Table~\ref{table: llama8B auroc auprc}, respectively. The proposed method consistently outperforms the baselines across both metrics, demonstrating stronger discriminative capability in separating correct from incorrect predictions. This improvement aligns with the gains observed under the conformal abstention framework, further confirming that the proposed uncertainty measure yields a more reliable ranking of predictions.

\begin{table*}[t]
\centering
\small
\captionsetup{justification=justified}
\caption{AUROC and AUPRC of different uncertainty scores with LLaMA-3.2-3B-Instruct. For each dataset, the highest value is shown in \textbf{bold}.}
\vspace{-5pt}
\setlength{\tabcolsep}{2pt}
\begin{tabularx}{\textwidth}{c|YYYYYY|YYYYYY} 
\specialrule{0.4pt}{0pt}{2pt}
\specialrule{0.4pt}{0pt}{3pt}
\multirow{2}{*}{Method} & \multicolumn{6}{c|}{AUROC} & \multicolumn{6}{c}{AUPRC} \\
& NQ & TQA & SQW & SciQ & GSM8K & CQA & NQ & TQA & SQW & SciQ & GSM8K & CQA \\
\midrule

Perplexity      & 0.71 & 0.64 & 0.76 & 0.75 & 0.77 & 0.69 & 0.67 & 0.51 & 0.81 & 0.92 & 0.40 &  0.87\\
SAR             & \textbf{0.77} & 0.70 & 0.77 & 0.76 & 0.73 & 0.69 & 0.74 & 0.52 & 0.81 & 0.93 & 0.39 &  0.88\\
Semantic & 0.64 & 0.48 & 0.77 & 0.80 & 0.73 & 0.66 & 0.62 & 0.36 & 0.82 & 0.94 & 0.32 &  0.87\\
Focus           & 0.62 & 0.50 & 0.55 & 0.64 & 0.56 & 0.69 & 0.64 & 0.40 & 0.67 & 0.87 & 0.20 &  0.86\\
Attention  & 0.50 & 0.53 & 0.45 & 0.52 & 0.39 & 0.50 & 0.53 & 0.40 & 0.54 & 0.81 & 0.11 &  0.77\\
Eigen      & 0.63 & 0.42 & 0.76 & 0.79 & 0.65 & 0.68 & 0.62 & 0.33 & 0.80 & 0.94 & 0.31 &  0.87\\
ATRMD         & 0.59 & 0.45 & 0.53 & 0.55 & 0.56 & 0.53 & 0.58 & 0.40 & 0.59 & 0.85 & 0.16 & 0.80 \\
P(True)         & 0.55 & 0.40 & 0.49 & 0.51 & 0.46 & 0.47 & 0.55 & 0.32 & 0.55 & 0.81 & 0.14 & 0.76 \\
Proposed        & \textbf{0.77} &\textbf{0.88} & \textbf{0.81} & \textbf{0.82} & \textbf{0.78} & \textbf{0.74} & \textbf{0.75} & \textbf{0.87} & \textbf{0.86} & \textbf{0.95} & \textbf{0.42} & \textbf{0.90}\\
\bottomrule
\end{tabularx}
\label{table: llama3B auroc auprc}
\end{table*}
\begin{table*}[t]
\centering
\small
\captionsetup{justification=justified}
\caption{AUROC and AUPRC of different uncertainty scores with LLaMA-3-8B-Instruct. For each dataset, the highest value is shown in \textbf{bold}.}
\vspace{-5pt}
\setlength{\tabcolsep}{2pt}
\begin{tabularx}{\textwidth}{c|YYYYYY|YYYYYY} 
\specialrule{0.4pt}{0pt}{2pt}
\specialrule{0.4pt}{0pt}{3pt}
\multirow{2}{*}{Method} & \multicolumn{6}{c|}{AUROC} & \multicolumn{6}{c}{AUPRC} \\
& NQ & TQA & SQW & SciQ & GSM8K & CQA & NQ & TQA & SQW & SciQ & GSM8K & CQA \\
\midrule

Perplexity      & 0.67 & 0.66 & 0.70 & 0.67 & 0.60 & 0.68 & 0.73 & 0.51 & 0.80 & 0.95 & 0.67 & 0.88\\
SAR             & 0.75 & 0.67 & 0.74 & 0.72 & 0.62 & 0.68 & 0.80 & 0.60 & 0.82 & 0.95 & 0.70 & 0.88\\
Semantic        & 0.63 & 0.52 & 0.74 & 0.67 & 0.61 & 0.68 & 0.68 & 0.39 & 0.83 & 0.95 & 0.66 & 0.88 \\
Focus           & 0.64 & 0.66 & 0.49 & 0.59 & 0.52 & 0.67 & 0.69 & 0.52 & 0.69 & 0.92 & 0.63 & 0.86 \\
Attention       & 0.49 & 0.50 & 0.47 & 0.53 & 0.46 & 0.54 & 0.57 & 0.45 & 0.63 & 0.91 & 0.55 & 0.81 \\
Eigen           & 0.57 & 0.50 & 0.70 & 0.70 & 0.61 & 0.69 & 0.64 & 0.37 & 0.80 & 0.95 & 0.69 & 0.88 \\
ATRMD         & 0.54 & 0.58 & 0.71 & 0.64 & 0.58 & 0.63 & 0.62 & 0.41 & 0.83 & 0.94 & 0.64 & 0.85 \\
P(True)         & 0.61 & 0.65 & 0.47 & 0.60 & 0.44 & 0.53 & 0.65 & 0.55 & 0.64 & 0.93 & 0.55 & 0.82 \\
Proposed        & \textbf{0.77} & \textbf{0.80} & \textbf{0.83} & \textbf{0.80} & \textbf{0.68} & \textbf{0.71} & \textbf{0.81} & \textbf{0.67} & \textbf{0.91} & \textbf{0.97} & \textbf{0.77} & \textbf{0.90} \\
\bottomrule
\end{tabularx}
\label{table: llama8B auroc auprc}
\end{table*}

\section{Performance of Calibrated Baselines}\label{appendix: effect of calibration}
The use of a calibrator in our framework raises a natural question: are the observed performance gains driven by the proposed uncertainty metrics themselves, or simply by the additional post-processing step? To disentangle these effects, we design an evaluation that explicitly isolates the contribution of calibration. Specifically, we augment baseline methods, whose original formulations do not involve any training, with the same calibration procedure employed in our approach. For each baseline, we report its calibrated output in Table~\ref{table: gemma auroc auprc calibrted baselines}, Table~\ref{table: llama3B auroc auprc calibrated baseline}, and Table~\ref{table: llama8B auroc auprc calibrated baseline}. We also include the performance of our method in these tables to facilitate direct comparison.

The results indicate that calibration does not necessarily improve these baselines; in fact, their calibrated variants often underperform the raw scores, suggesting susceptibility to overfitting. This behavior is consistent with the nature of these methods. These baselines produce a single scalar uncertainty measure that is already designed to exhibit a monotonic relationship with prediction error. As such, their discriminative capacity is largely intrinsic and requires minimal post-hoc adjustment. Introducing a learned calibrator on top of these one-dimensional signals provides little additional information and can instead amplify noise, likely degrading performance.

In contrast, our method benefits from calibration because it operates on richer, higher-dimensional geometry signals whose relationship to correctness is not trivially monotonic, allowing the calibrator to meaningfully refine the decision boundary.

\begin{table}[h]
\centering
\small
\captionsetup{justification=justified}
\caption{AUROC and AUPRC of calibrated baselines compared with the proposed method, using Gemma-3-4B-Instruct as the base inference model. Arrows next to calibrated baseline results indicate whether performance improves or degrades compared to their original versions, while a dash denotes no change. For each dataset, the highest value is highlighted in \textbf{bold}.}
\setlength{\tabcolsep}{2pt}
\begin{tabularx}{\textwidth}{c|YYYYYY|YYYYYY} 
\specialrule{0.4pt}{0pt}{2pt}
\specialrule{0.4pt}{0pt}{3pt}
\multirow{2}{*}{Method} & \multicolumn{6}{c|}{AUROC} & \multicolumn{6}{c}{AUPRC} \\
& NQ & TQA & SQW & SciQ & GSM8K & CQA & NQ & TQA & SQW & SciQ & GSM8K & CQA \\
\midrule

Perplexity      & 0.67 $\textcolor{red}{\downarrow}$ & 0.50 $\textcolor{red}{\downarrow}$& 0.71 --           & 0.71 $\textcolor{green!80!black}{\uparrow}$  & 0.52 $\textcolor{red}{\downarrow}$ & 0.70 $\textcolor{red}{\downarrow}$ & 0.54$\textcolor{red}{\downarrow}$& 0.36$\textcolor{red}{\downarrow}$ & 0.63  --          & 0.91 --          & 0.40 $\textcolor{red}{\downarrow}$& 0.88$\textcolor{red}{\downarrow}$ \\
SAR             & 0.79 $\textcolor{red}{\downarrow}$ & 0.62 --            & 0.74 $\textcolor{green!80!black}{\uparrow}$ & 0.74 --            & 0.48 $\textcolor{red}{\downarrow}$ & 0.70 $\textcolor{green!80!black}{\uparrow}$   & 0.66$\textcolor{red}{\downarrow}$& 0.43$\textcolor{red}{\downarrow}$ & 0.67$\textcolor{green!80!black}{\uparrow}$  & 0.92 $\textcolor{green!80!black}{\uparrow}$   & 0.38 $\textcolor{red}{\downarrow}$& 0.88$\textcolor{green!80!black}{\uparrow}$ \\
Semantic        & 0.64 $\textcolor{green!80!black}{\uparrow}$   & 0.41 $\textcolor{red}{\downarrow}$& 0.73 --           & 0.74 $\textcolor{red}{\downarrow}$& 0.51 $\textcolor{red}{\downarrow}$ & 0.70 --             & 0.52$\textcolor{red}{\downarrow}$& 0.33$\textcolor{red}{\downarrow}$ & 0.67$\textcolor{green!80!black}{\uparrow}$  & 0.91 --             & 0.38 $\textcolor{red}{\downarrow}$& 0.88-- \\
Focus           & 0.69  --            & 0.54 $\textcolor{red}{\downarrow}$& 0.66 $\textcolor{green!80!black}{\uparrow}$ & 0.66 $\textcolor{green!80!black}{\uparrow}$  & 0.52 $\textcolor{green!80!black}{\uparrow}$   & 0.59 $\textcolor{red}{\downarrow}$ & 0.57$\textcolor{green!80!black}{\uparrow}$  & 0.37$\textcolor{red}{\downarrow}$ & 0.60$\textcolor{green!80!black}{\uparrow}$  & 0.88 $\textcolor{red}{\downarrow}$ & 0.39 $\textcolor{red}{\downarrow}$& 0.80$\textcolor{red}{\downarrow}$ \\
Attention       & 0.55 $\textcolor{red}{\downarrow}$ & 0.70 $\textcolor{green!80!black}{\uparrow}$  & 0.54 $\textcolor{green!80!black}{\uparrow}$ & 0.50 $\textcolor{red}{\downarrow}$& 0.54 $\textcolor{green!80!black}{\uparrow}$   & 0.52 --             & 0.44$\textcolor{red}{\downarrow}$& 0.56$\textcolor{green!80!black}{\uparrow}$   & 0.49$\textcolor{green!80!black}{\uparrow}$  & 0.83 $\textcolor{red}{\downarrow}$ & 0.40 $\textcolor{green!80!black}{\uparrow}$  & 0.78 $\textcolor{green!80!black}{\uparrow}$ \\
Eigen           & 0.57 $\textcolor{red}{\downarrow}$ & 0.50 $\textcolor{green!80!black}{\uparrow}$  & 0.74 $\textcolor{green!80!black}{\uparrow}$ & 0.74 $\textcolor{green!80!black}{\uparrow}$  & 0.56--              & 0.61 $\textcolor{red}{\downarrow}$ & 0.50$\textcolor{red}{\downarrow}$& 0.36$\textcolor{green!80!black}{\uparrow}$   & 0.70$\textcolor{green!80!black}{\uparrow}$  & 0.91 $\textcolor{green!80!black}{\uparrow}$   & 0.44 --            & 0.81 $\textcolor{green!80!black}{\uparrow}$ \\
P(True)         & 0.51 $\textcolor{green!80!black}{\uparrow}$   & 0.56 $\textcolor{red}{\downarrow}$& 0.56  --          & 0.50 $\textcolor{red}{\downarrow}$& 0.51 $\textcolor{red}{\downarrow}$ & 0.56 $\textcolor{red}{\downarrow}$ & 0.43$\textcolor{green!80!black}{\uparrow}$  & 0.41$\textcolor{red}{\downarrow}$ & 0.49$\textcolor{red}{\downarrow}$& 0.83$\textcolor{red}{\downarrow}$  & 0.39  --           & 0.79 $\textcolor{red}{\downarrow}$ \\
Proposed        & \textbf{0.83} & \textbf{0.86} & \textbf{0.82} & \textbf{0.81} & \textbf{0.71} & \textbf{0.72} & \textbf{0.78} & \textbf{0.82} & \textbf{0.81} & \textbf{0.95} & \textbf{0.60} & \textbf{0.90} \\
\bottomrule
\end{tabularx}
\label{table: gemma auroc auprc calibrted baselines}
\end{table}

\begin{table}[h]
\centering
\small
\captionsetup{justification=justified}
\caption{AUROC and AUPRC of calibrated baselines compared with the proposed method, using LLaMA-3.2-3B-Instruct as the base inference model. Arrows next to calibrated baseline results indicate whether performance improves or degrades compared to their original versions, while a dash denotes no change. For each dataset, the highest value is highlighted in \textbf{bold}.}
\setlength{\tabcolsep}{2pt}
\begin{tabularx}{\textwidth}{c|YYYYYY|YYYYYY} 
\specialrule{0.4pt}{0pt}{2pt}
\specialrule{0.4pt}{0pt}{3pt}
\multirow{2}{*}{Method} & \multicolumn{6}{c|}{AUROC} & \multicolumn{6}{c}{AUPRC} \\
& NQ & TQA & SQW & SciQ & GSM8K & CQA & NQ & TQA & SQW & SciQ & GSM8K & CQA \\
\midrule

Perplexity      & 0.69 $\textcolor{red}{\downarrow}$ & 0.61 $\textcolor{red}{\downarrow}$ & 0.77 $\textcolor{green!80!black}{\uparrow}$ & 0.75 --         & 0.76$\textcolor{red}{\downarrow}$ & 0.68 $\textcolor{red}{\downarrow}$  & 0.64 $\textcolor{red}{\downarrow}$ & 0.44 $\textcolor{red}{\downarrow}$ & 0.81 --              & 0.92 --          & 0.37 $\textcolor{red}{\downarrow}$ &  0.87 --\\
SAR             & 0.76 $\textcolor{red}{\downarrow}$ & 0.68 $\textcolor{red}{\downarrow}$ & 0.77 --         & 0.76 --         & 0.72$\textcolor{red}{\downarrow}$ & 0.69 --            & 0.75 $\textcolor{green!80!black}{\uparrow}$   & 0.54 $\textcolor{green!80!black}{\uparrow}$ & 0.81 --                & 0.93 --          & 0.37 $\textcolor{red}{\downarrow}$ &  0.87$\textcolor{red}{\downarrow}$\\
Semantic        & 0.65 $\textcolor{green!80!black}{\uparrow}$   & 0.72 $\textcolor{green!80!black}{\uparrow}$   & 0.77 --         & 0.80 --         & 0.73 --          & 0.66 --            & 0.63 $\textcolor{red}{\downarrow}$ & 0.53 $\textcolor{green!80!black}{\uparrow}$ & 0.81 $\textcolor{red}{\downarrow}$ & 0.94 --          & 0.30 $\textcolor{red}{\downarrow}$ &  0.87--\\
Focus           & 0.66 $\textcolor{green!80!black}{\uparrow}$   & 0.55 $\textcolor{green!80!black}{\uparrow}$   & 0.66 $\textcolor{green!80!black}{\uparrow}$ & 0.65 $\textcolor{green!80!black}{\uparrow}$ & 0.58$\textcolor{green!80!black}{\uparrow}$   & 0.69 --            & 0.62 $\textcolor{red}{\downarrow}$ & 0.42 $\textcolor{green!80!black}{\uparrow}$ & 0.68 $\textcolor{green!80!black}{\uparrow}$  & 0.87 --           & 0.17 $\textcolor{red}{\downarrow}$ &  0.87$\textcolor{green!80!black}{\uparrow}$\\
Attention       & 0.51 $\textcolor{red}{\downarrow}$ & 0.57 $\textcolor{green!80!black}{\uparrow}$   & 0.50 $\textcolor{green!80!black}{\uparrow}$ & 0.52 --         & 0.59$\textcolor{green!80!black}{\uparrow}$   & 0.53 $\textcolor{green!80!black}{\uparrow}$    & 0.52 $\textcolor{red}{\downarrow}$ & 0.43 $\textcolor{green!80!black}{\uparrow}$& 0.57 $\textcolor{green!80!black}{\uparrow}$  & 0.81 --           & 0.18 $\textcolor{green!80!black}{\uparrow}$ &  0.79 $\textcolor{green!80!black}{\uparrow}$\\
Eigen           & 0.63 --           & 0.49 $\textcolor{green!80!black}{\uparrow}$   & 0.76 --         & 0.78$\textcolor{red}{\downarrow}$& 0.66$\textcolor{green!80!black}{\uparrow}$   & 0.68 --            & 0.59 $\textcolor{red}{\downarrow}$& 0.39 $\textcolor{green!80!black}{\uparrow}$ & 0.81 $\textcolor{green!80!black}{\uparrow}$  & 0.93 $\textcolor{red}{\downarrow}$ & 0.31 $\textcolor{red}{\downarrow}$ &  0.87--\\
P(True)         & 0.53 $\textcolor{red}{\downarrow}$ & 0.58 $\textcolor{green!80!black}{\uparrow}$   & 0.50 $\textcolor{green!80!black}{\uparrow}$ & 0.49$\textcolor{red}{\downarrow}$& 0.50$\textcolor{green!80!black}{\uparrow}$   & 0.50 $\textcolor{green!80!black}{\uparrow}$    & 0.53 $\textcolor{red}{\downarrow}$ & 0.42 $\textcolor{green!80!black}{\uparrow}$ & 0.57 $\textcolor{green!80!black}{\uparrow}$  & 0.80 $\textcolor{red}{\downarrow}$ & 0.15 $\textcolor{green!80!black}{\uparrow}$ & 0.77 $\textcolor{green!80!black}{\uparrow}$ \\
Proposed        & \textbf{0.77}              & \textbf{0.88}              & \textbf{0.81}            & \textbf{0.82}            & \textbf{0.78}             & \textbf{0.74} & \textbf{0.75} & \textbf{0.87} & \textbf{0.86} & \textbf{0.95} & \textbf{0.42} &  \textbf{0.90}\\
\bottomrule
\end{tabularx}
\label{table: llama3B auroc auprc calibrated baseline}
\end{table}

\begin{table}[h]
\centering
\small
\captionsetup{justification=justified}
\caption{AUROC and AUPRC of calibrated baselines compared with the proposed method, using LLaMA-3-8B-Instruct as the base inference model. Arrows next to calibrated baseline results indicate whether performance improves or degrades compared to their original versions, while a dash denotes no change. For each dataset, the highest value is highlighted in \textbf{bold}.}
\setlength{\tabcolsep}{2pt}
\begin{tabularx}{\textwidth}{c|YYYYYY|YYYYYY} 
\specialrule{0.4pt}{0pt}{2pt}
\specialrule{0.4pt}{0pt}{3pt}
\multirow{2}{*}{Method} & \multicolumn{6}{c|}{AUROC} & \multicolumn{6}{c}{AUPRC} \\
& NQ & TQA & SQW & SciQ & GSM8K & CQA & NQ & TQA & SQW & SciQ & GSM8K & CQA \\
\midrule

Perplexity      & 0.66 $\textcolor{red}{\downarrow}$ & 0.52 $\textcolor{red}{\downarrow}$ & 0.70 --         & 0.67 --           & 0.50 $\textcolor{red}{\downarrow}$& 0.68 --          & 0.70$\textcolor{red}{\downarrow}$ & 0.48$\textcolor{red}{\downarrow}$ & 0.80 --                   & 0.94 $\textcolor{red}{\downarrow}$ & 0.59 $\textcolor{red}{\downarrow}$ &  0.87$\textcolor{red}{\downarrow}$\\
SAR             & 0.74 $\textcolor{red}{\downarrow}$ & 0.63 $\textcolor{red}{\downarrow}$ & 0.74 --         & 0.72 --           & 0.67 $\textcolor{green!80!black}{\uparrow}$  & 0.68 --          & 0.77$\textcolor{red}{\downarrow}$ & 0.51$\textcolor{red}{\downarrow}$ & 0.82 --                   & 0.95 --           & 0.63 $\textcolor{red}{\downarrow}$ &  0.89$\textcolor{green!80!black}{\uparrow}$\\
Semantic        & 0.62 $\textcolor{red}{\downarrow}$ & 0.44 $\textcolor{red}{\downarrow}$ & 0.74 --         & 0.65 $\textcolor{red}{\downarrow}$ & 0.59 $\textcolor{red}{\downarrow}$& 0.68 --          & 0.67$\textcolor{red}{\downarrow}$ & 0.33$\textcolor{red}{\downarrow}$ & 0.83 --                   & 0.95 --           & 0.66 --           & 0.89 $\textcolor{green!80!black}{\uparrow}$ \\
Focus           & 0.66 $\textcolor{green!80!black}{\uparrow}$   & 0.61 $\textcolor{red}{\downarrow}$ & 0.57 $\textcolor{green!80!black}{\uparrow}$ & 0.58 $\textcolor{red}{\downarrow}$ & 0.55 $\textcolor{green!80!black}{\uparrow}$  & 0.66 $\textcolor{red}{\downarrow}$& 0.69--           & 0.47$\textcolor{red}{\downarrow}$ & 0.70$\textcolor{green!80!black}{\uparrow}$  & 0.93 $\textcolor{green!80!black}{\uparrow}$   & 0.62 $\textcolor{red}{\downarrow}$ & 0.86 --\\
Attention       & 0.49 --           & 0.46 $\textcolor{red}{\downarrow}$ & 0.53 $\textcolor{green!80!black}{\uparrow}$ & 0.54 $\textcolor{green!80!black}{\uparrow}$   & 0.53 $\textcolor{green!80!black}{\uparrow}$  & 0.50 $\textcolor{red}{\downarrow}$& 0.56$\textcolor{red}{\downarrow}$ & 0.35$\textcolor{red}{\downarrow}$ & 0.67 $\textcolor{green!80!black}{\uparrow}$      & 0.91 --           & 0.61 $\textcolor{green!80!black}{\uparrow}$ & 0.79 $\textcolor{red}{\downarrow}$ \\
Eigen           & 0.58 $\textcolor{green!80!black}{\uparrow}$   & 0.51 $\textcolor{green!80!black}{\uparrow}$   & 0.71 $\textcolor{green!80!black}{\uparrow}$ & 0.69 $\textcolor{red}{\downarrow}$ & 0.60 $\textcolor{red}{\downarrow}$& 0.68 $\textcolor{red}{\downarrow}$& 0.62$\textcolor{red}{\downarrow}$ & 0.36$\textcolor{red}{\downarrow}$ & 0.81 $\textcolor{green!80!black}{\uparrow}$      & 0.95 --           & 0.68 $\textcolor{red}{\downarrow}$ & 0.89 $\textcolor{green!80!black}{\uparrow}$ \\
P(True)         & 0.61 --           & 0.50 $\textcolor{red}{\downarrow}$ & 0.53 $\textcolor{green!80!black}{\uparrow}$ & 0.56 $\textcolor{red}{\downarrow}$ & 0.52 $\textcolor{green!80!black}{\uparrow}$  & 0.55 $\textcolor{green!80!black}{\uparrow}$  & 0.53$\textcolor{red}{\downarrow}$ & 0.37$\textcolor{red}{\downarrow}$ & 0.68 $\textcolor{green!80!black}{\uparrow}$      & 0.91 $\textcolor{red}{\downarrow}$ & 0.61 $\textcolor{green!80!black}{\uparrow}$ & 0.82 --\\
Proposed        & \textbf{0.77} & \textbf{0.80} & \textbf{0.83} & \textbf{0.80} & \textbf{0.68} & \textbf{0.71} & \textbf{0.81} & \textbf{0.67} & \textbf{0.91} & \textbf{0.97} & \textbf{0.77} & \textbf{0.90} \\
\bottomrule
\end{tabularx}
\label{table: llama8B auroc auprc calibrated baseline}
\end{table}
\clearpage
\section{Memory and Computational Complexity}
Let $T$ denote the sequence length, $d$ the hidden dimension, and $L$ the number of Transformer blocks. Since our framework extracts features from both attention and MLP modules, the total number of considered layers is $2L-1$.

\paragraph{Memory Complexity.}
At each layer $l$, we construct two types of contribution matrices. The MLP contribution forms a diagonal $T \times T$ matrix $C_{\mathrm{MLP}}^l$, which requires $\mathcal{O}(T)$ memory, while the attention contribution $C_{\mathrm{Attn}}^l$ corresponds to a lower-triangular $T \times T$ matrix, requiring $\mathcal{O}(T^2)$ memory. Therefore, the per-layer memory cost is dominated by $\mathcal{O}(T^2)$, and across layers becomes $\mathcal{O}(L T^2)$.

In addition, we only retain hidden states associated with the final input token across layers, resulting in a memory cost of $\mathcal{O}(L d)$. Other derived quantities, such as embedding rotations and angular features, are also computed on the final token and incur at most $\mathcal{O}(L d)$ memory.

Combining these terms, the total memory complexity is
\begin{equation}
\mathcal{O}(L T^2 + L d).
\end{equation}

\textbf{Computational complexity.}
The dominant cost arises from computing the proximity function
\begin{equation}
\mathrm{prox}(z, z_i) = \frac{\max\left(0, \|z\|_1 - \|z - z_i\|_1\right)}{\sum_{j=1}^{T} \max\left(0, \|z\|_1 - \|z - z_j\|_1\right)},
\end{equation}
where $z, z_i \in \mathbb{R}^d$. Computing $\|z - z_i\|_1$ requires $\mathcal{O}(d)$ operations. For each layer, the attention contribution involves $\mathcal{O}(T^2)$ token pairs, resulting in $\mathcal{O}(T^2 d)$ time. The MLP contribution only involves diagonal terms and is therefore $\mathcal{O}(T d)$, which is negligible compared to the attention term.

Other computations, including embedding rotation and angular distances to in-distribution and out-of-distribution directions, scale as $\mathcal{O}(T d)$ per layer. Therefore, the overall computational complexity is dominated by the pairwise proximity computation:
\begin{equation}
\mathcal{O}(L T^2 d).
\end{equation}

\textbf{Summary.}
Overall, the method scales quadratically with sequence length and linearly with both hidden dimension and number of layers. The pairwise proximity computation constitutes the primary computational bottleneck.

\end{document}